\documentclass[onefignum,onetabnum]{siamonline250211}

\usepackage{lipsum}
\usepackage{amsfonts}
\usepackage{graphicx}
\usepackage{epstopdf}
\usepackage{algorithmic}
\ifpdf
  \DeclareGraphicsExtensions{.eps,.pdf,.png,.jpg}
\else
  \DeclareGraphicsExtensions{.eps}
\fi

\newsiamremark{remark}{Remark}
\newsiamremark{hypothesis}{Hypothesis}
\crefname{hypothesis}{Hypothesis}{Hypotheses}
\newsiamthm{claim}{Claim}

\newtheorem{exmp}{Example}

\title{Lipschitz bounds for integral kernels\thanks{Under review.
\funding{This work was partially funded by LIEBHERR Aerospace Toulouse. SZ acknowledges the support 
by the ANR LabEx CIMI (grant ANR-11-LABX-0040) within the French State
Programme “Investissements d’Avenir.”}}}

\author{
Justin Reverdi\thanks{IRIT, IMT, LIEBHERR Aerospace, Toulouse, France
  (\email{justin.reverdi@gmail.com}).}
\and
Sixin Zhang\thanks{Univ Toulouse, Toulouse INP, CNRS, IRIT, Toulouse, France
  (\email{sixin.zhang@irit.fr}, \email{serge.gratton@irit.fr}).}
\and
Fabrice Gamboa\thanks{IMT, Toulouse, France
  (\email{fabrice.gamboa@math.univ-toulouse.fr}).}
\and
Serge Gratton\footnotemark[3]
  }
\usepackage{amsopn}

\makeatletter
\newcommand*{\addFileDependency}[1]{
  \typeout{(#1)}
  \@addtofilelist{#1}
  \IfFileExists{#1}{}{\typeout{No file #1.}}
}
\makeatother

\usepackage{amsmath,amsfonts,bm}

\def\eqref#1{equation~\ref{#1}}

\def\1{\bm{1}}

\DeclareMathAlphabet{\mathsfit}{\encodingdefault}{\sfdefault}{m}{sl}
\SetMathAlphabet{\mathsfit}{bold}{\encodingdefault}{\sfdefault}{bx}{n}

\newcommand{\E}{\mathbb{E}}

\newcommand{\R}{\mathbb{R}}

\newcommand{\Cov}{\mathrm{Cov}}

\usepackage{chngcntr}

\usepackage{ntheorem}
\usepackage{mathrsfs}  
\usepackage{tikz-cd}
\usepackage{subcaption}
\usepackage[margin=1in]{geometry}
\usepackage{hyperref}
\usepackage{amssymb}
\usepackage{booktabs}
\usepackage{amsmath}
\usepackage{mathtools}
\usepackage{fullpage}
\allowdisplaybreaks
\usepackage{url}
\usepackage{dsfont}
\usepackage{enumerate}

\usepackage{bbm}

\theoremstyle{remark}

\newcommand{\RR}{\mathbb{R}}
\newcommand{\Lip}{\text{Lip}}

\usepackage[normalem]{ulem}
\newcommand{\stkout}[1]{\ifmmode\text{\sout{\ensuremath{#1}}}\else\sout{#1}\fi}

\newcommand{\X}{\mathcal{X}}
\newcommand{\HH}{{\mathcal{H}_k}}
\newcommand{\om}{\omega}
\newcommand{\thet}{\theta}

\newcommand{\vphi}{\varphi_k}

\newcommand{\Om}{\Omega}
\newcommand{\op}{\textit{op}}
\newcommand{\Span}{\textit{Span}}

\newcommand{\lV}{\left\lVert}
\newcommand{\rV}{\right\rVert}
\newcommand{\rVH}{\rV_\HH}
\newcommand{\rVE}{\rV_E}
\newcommand{\rVe}{\rV}
\newcommand{\rVop}{\rV_{\op}}

\newcommand\Lp[2]{{L^{#2}(#1)}}
\newcommand\LLp[2]{{\mathcal{L}^{#2}(#1)}}
\newcommand\rVL[2]{\rV_\Lp{#1}{#2}}

\newcommand\LpX[2]{{L_\X^{#2}(#1)}}

\newcommand{\lb}{\left\langle}
\newcommand{\rb}{\right\rangle}
\newcommand{\rbH}{\rb_\HH}
\newcommand{\rbL}[2]{\rb_\Lp{#1}{#2}}
\newcommand{\rbLL}[2]{\rb_\LLp{#1}{#2}}

\newcommand{\T}{\top}
\newcommand{\uu}{u}
\newcommand{\uv}{v}
\newcommand{\ha}{h}

\newcommand{\z}{z}
\newcommand{\BE}{{{\bar{B}_E}}}
\newcommand{\BR}{{{\bar{B}_{\mathbb{R}^d}}}}

\newcommand{\EP}{{\mathbb{E}_P}}
\newcommand{\EPP}{{\mathbb{E}_{\zeta,b\sim\mu^1_\gamma\otimes p_b}}}

\ifpdf
\hypersetup{
  pdftitle={Lipschitz constant of the feature map of integral kernels},
  pdfauthor={}
}

\begin{document}

\maketitle

\begin{abstract}
Feature maps associated with positive definite kernels play a central role in kernel methods and learning theory, where regularity properties such as Lipschitz continuity are closely related to robustness and stability guarantees. Despite their importance, explicit characterizations of the Lipschitz constant of kernel feature maps are available only in a limited number of cases. In this paper, we study the Lipschitz regularity of feature maps associated with integral kernels under differentiability assumptions. We first provide sufficient conditions ensuring Lipschitz continuity and derive explicit formulas for the corresponding Lipschitz constants. We then identify a condition under which the feature map fails to be Lipschitz continuous and apply these results to several important classes of kernels. For infinite width two-layer neural network with isotropic Gaussian weight distributions, we show that the Lipschitz constant of the associated kernel can be expressed as the supremum of a two-dimensional integral, leading to an explicit characterization for the Gaussian kernel and the ReLU random neural network kernel. We also study continuous and shift-invariant kernels such as Gaussian, Laplace, and Matérn kernels, which admit an interpretation as neural network with cosine activation function. In this setting, we prove that the feature map is Lipschitz continuous if and only if the weight distribution has a finite second-order moment, and we then derive its Lipschitz constant. Finally, we raise an open question concerning the asymptotic behavior of the convergence of the Lipschitz constant in finite width neural networks. Numerical experiments are provided to support this behavior.
\end{abstract}


\vspace{0.6cm}

\section{Introduction}

The deployment of learning algorithms in safety–critical environments demands
quantitative guaranties of stability and controlled sensitivity to input
perturbations. 
Robustness to small adversarial perturbations has emerged as a
central requirement in modern applications such as autonomous decision systems,
scientific imaging, and medical diagnostics \cite{SIAM1, SIAM2,SIAM3}. 
A classical and powerful global
measure of robustness is the Lipschitz constant of a predictor $h : \X \to \R $ on the input space $\X$.
Bounding this
quantity ensures that the model cannot react excessively to small variations of
the input, thus providing a form of certified robustness against various perturbations, including adversarial
attacks \cite{SIAM1,combettes_lipschitz_2020,SIAM3}. 
In adversarial robustness, a global Lipschitz constraint upper-bounds the worst-case amplification of perturbations and therefore yields a certified radius of resistance to attacks \cite{SIAM1}.  
In this context, designing methods that
allow one to compute or tightly bound the Lipschitz constant of a model is of fundamental importance.

For kernel methods, robustness bounds often reduce to bounding the Lipschitz constant of the underlying representation.  
Recent work has investigated this question based on the same central idea: to analyze the Lipschitz continuity of the feature map 
associated with a given kernel \cite{fiedler_lipschitz_2023, lipop, D2KE, NIPS2017_38ed162a, bietti_inductive_2019, vanwaarde2021kernelbasedmodelsanalysis}. 
More precisely, 
on a Hilbert space $\HH$ (the RKHS of $k$) consisting of functions from $\X$ to $\RR$, one analyzes 
the Lipschitz constant of the associated feature map $\vphi:x\mapsto k(x,\cdot)$, and then applies 
the Cauchy-Schwarz inequality 
to obtain an upper bound for every $h\in\HH$, that 
is, $\Lip(\ha)\leq \Lip(\vphi)\lV \ha \rVH$
thanks to the reproducing property $h(x) = \lb h, k(x,\cdot) \rbH$. 
Following this line of research, 
our main contribution is to obtain a closed-form 
characterization of the Lipschitz constant 
of $\vphi$ when $k$ belongs to integral kernels \cite{integral_rkhs, fiedler_lipschitz_2023, rahimi_random_nodate}. This family of kernels is related to one-layer neural networks with random features for large-scale kernel classification and regression \cite{rahimi_random_nodate, rf1, rf2, rfnn}. 

Section \ref{sec:diffink} reviews the notion of integral kernel, which admits an integral representation of the form
\[
k(x,x')
= \int_{\Om} \phi(\om,x)\,\phi(\om,x')\,dP(\om)~.
\]
Here, $\phi$ plays the role of random features, and they are averaged through a probability measure $P$. 
It includes a broad and expressive class of kernels encountered in applications \cite{SIAM4, D2KE, rahimi_random_nodate, integral_rkhs}.
Their regularity has been recently studied aiming to get Lipschitz bounds on every functions of their associated RKHS \cite{fiedler_lipschitz_2023, D2KE, bietti_approximation_2022, BNN}.
However, existing Lipschitz bounds for the associated feature maps are typically non-exact : they rely on uniform bounds on $\phi(\om,\cdot)$ with respect to $\om$, or provide only coarse inequalities.
The present work addresses this gap by deriving an exact formula for the
Lipschitz constant $\Lip(\vphi)$ of an integral kernel $k$ when $\phi$ admits a differential structure. This is the smallest possible Lipschitz bound for the feature map $\vphi$.

We provide in Section \ref{sec:appRNN} an exact formula for $\Lip(\vphi)$ when
$\phi(\omega,x)=\sigma(\langle w,x\rangle + b)$, with an isotropic Gaussian law on $w$.
We find that the exact formula for $\Lip(\vphi)$ reduces to a two-dimensional integral
which becomes computationally tractable. Section \ref{sec:RFF} analyzes continuous and shift-invariant kernels, 
including the common Gaussian and Laplacian kernels. We further discuss conditions where $\Lip(\vphi) = +\infty$.
Our results extend and substantially sharpen classical results on Lipschitz continuity in RKHS, 
thereby serving as a basis for a principled robustness
analysis in kernel-based models.
Section \ref{sec:num} studies numerically the behavior of the Lipschitz constant 
of two-layer neural networks with a finite number of random features, in relation 
to the obtained exact formula. They provide an empirical support regarding the  
convergence of this Lipschitz constant to $\Lip(\vphi)$ as the number of random features grows to infinity.
Section \ref{sec:con} concludes.

\textbf{Notations}: Let $(\Omega,\mathcal{A},P)$ denote a probability space. Let $(\Lp{P}{2},\lb\cdot,\cdot\rbL{P}{2})$ be the Hilbert space of real-valued, square-integrable functions on $\Omega$, with the standard inner product $\lb \uu,\uv \rbL{P}{2}$ (see Appendix \ref{Appendix:Lp}).
For two metric spaces $(A,d_A)$ and $(B,d_B)$, the Lipschitz constant $\Lip(f)$ of a Lipschitz-continuous function $f:A\to B$ is the smallest $C\geq 0$ such that
$d_B(f(a_2), f(a_1)) \leq C\,d_A(a_2, a_1), \forall a_1,a_2\in A.$ (see Appendix \ref{Appendix:Banach}).
Let $(E,\lV\cdot\rV_E)$ and $(F,\lV\cdot\rV_F)$ be Banach spaces. 
For a function $f:E\mapsto F$ which is Fréchet differentiable at $x_0$ (see Appendix \ref{Appendix:Banach}), we note $D_x f(x_0)$ its Fréchet derivative, or just $D f(x_0)$ when there is no ambiguity. The norm of a bounded linear operator $T:E\mapsto F$ is written as $\lV T \rVop$. Finally, the closed unit ball of $E$ is denoted by $\BE$.

\section{Differentiable integral kernels}
\label{sec:diffink}

The kernel methods in machine learning consist of finding solutions in a functional space called the RKHS (see \Cref{Appendix:kernel}). The purpose of this section is to study the Lipschitz continuity of these functions for a particular kind of kernel, called integral kernels. Based on a fundamental representation theorem for RKHS associated with an integral kernel in Section \ref{sec:preliminaryRKHSIkernel}, we derive in Section \ref{sec:diffIntKernel} an exact formula to compute the Lipschitz constant of their feature maps
under a suitable differentiability condition of integral kernels. 
This key result allows us to 
obtain an exact Lipschitz constant
for the feature map in infinite random neural networks and shift-invariant kernels in Section \ref{sec:appRNN} and \ref{sec:RFF}.

\subsection{Preliminaries on Lipschitz Continuity in RKHS and Integral Kernels}
\label{sec:preliminaryRKHSIkernel}

We first explain how to characterize the Lipschitz continuity of the functions in an RKHS defined on a metric space $\X$. Having a metric on $\X$ is necessary to define the Lipschitz continuity.
Then, we review the definition and properties of integral kernels which we focus on in the article. In particular, we expose an elementary proof of a fundamental representation theorem 
that characterizes the RKHS associated with integral kernels as mixtures with elements from $\Lp{P}{2}$. This theorem also provides an expression for the RKHS norm that will be useful in computing the exact Lipschitz constant in the next section.

If $k$ is a positive definite and symmetric (PDS) kernel, then from \cite{rkhs_base}, there exists a unique Hilbert space $\HH$ (the RKHS of $k$) consisting of functions from $\X$ to $\RR$, such that $k(x,\cdot)\in \HH$ and $h(x) = \lb h, k(x,\cdot) \rbH$ for all $x\in\X$ and $h\in\HH$. The associated feature map is $\vphi:x\mapsto k(x,\cdot)$.
Consider a feature map $\vphi$ that is Lipschitz with respect to the metric on $\HH$ induced by its inner product. Then, 
we can obtain an upper bound for every $h\in\HH$ (see \Cref{Appendix:Banach} for background knowledge),
\begin{equation}\label{eq:lipRKHSchar}
    \Lip(\ha)\leq \Lip(\vphi)\lV \ha \rVH~.
\end{equation}

Our main interest is a family of kernels, called integral kernels,
where we can obtain an explicit control of $\Lip(\vphi)$.
It represents a large variety of kernels, such as the continuous kernels with closed $\X$ in the Mercer theorem
(replace the series in  \cite[Theorem 1]{isometry} 
by an integral with counting measure), the continuous and shift invariant kernels with $\X=\RR^d$ \cite[1.4.2 Bochner's Theorem]{rudin1962fourier}. 
Moreover, in the random feature methods \cite{rahimi_random_nodate}, the kernel that is approximated 
by random features is always expressed as an expectation, so it is also an integral kernel. We can further construct new kernels by integration as in \cite{D2KE} and \cite{learning_kernel}.

\begin{definition}[Integral kernel]
\label{integral_kernel}
   Let $\LpX{P}{2}$ be the space of functions $\phi:\Om\times\X\rightarrow\RR$ such that for every $x\in\X$ we have $\phi(\cdot,x) :=\om\mapsto \phi(\om,x)\in\Lp{P}{2}$. 
    Then, the integral kernel  of $\phi\in\LpX{P}{2}$ is the following application,
    \begin{align*} 
    k:\X\times\X&\rightarrow\RR\\
        (x,x')&\mapsto \int_\Om \phi(\om,x)\phi(\om,x')dP(\om)  ~.
    \end{align*}
\end{definition}

An integral kernel $k$ is always positive definite and symmetric (PDS) \cite[Theorem 5.5]{fiedler_lipschitz_2023}. The integral kernels are defined on the inner product of $\Lp{P}{2}$. The previous remark, combined with the fact that $\Lp{P}{2}$ is a Hilbert space, leads to a nice characterization of its RKHS and the RKHS norm.

\begin{theorem}\cite[Proposition 1]{isometry}
\label{integralrkhs}
 Let $\phi\in\LpX{P}{2}$, $\theta(x) = \phi(\cdot,x)$ for all $x\in\X$ and $V := \overline{\Span(\thet(\X)})$. Define the operator $S:\Lp{P}{2}\rightarrow \mathbb{R}^\X$ by
\begin{align*}
 S(\uu)(x)&= \lb \uu ,\thet(x) \rbL{P}{2}~,
\end{align*}
Then, the restriction of $S$ to $V$ into $\HH$ 
is an isometric isomorphism.
\[
    \begin{tikzcd}
    {} & \X \arrow[dr, "\vphi"] \arrow[dl, "\thet"'] \\
    V \arrow[rr, Rightarrow, "S_{|V}"] && \HH
    \end{tikzcd}
\]
\end{theorem}

Theorem \ref{integralrkhs} characterizes $\HH$ through the feature space $V$ based on the operator $S$. 
This is natural because the kernel $k$ matches the inner product in the Hilbert space $\Lp{P}{2}$, so we can represent the RKHS of $k$ as the image of the operator $S$ that transforms the elements of $\Lp{P}{2}$ into elements of $\HH$.

The next property 
shows that the image of $\Lp{P}{2}$ by $S$ is actually the RKHS $\HH$ and therefore the norm on $\HH$ 
can be related to the norm on $\Lp{P}{2}$.
It is fundamental for analyzing the Lipschitz continuity of $\vphi$.

\begin{corollary}
\cite[Theorem 3.1 and Corollary 4.1]{integral_rkhs} 
\label{cor:integral_rkhs}
    Under the same assumption as Theorem \ref{integralrkhs}, we have
    $V = (\textit{Ker}(S))^\perp$ and the integral kernel $k$ of $\phi$ has the associated RKHS \[\HH = \left\{S(u) = y\mapsto\int_\Om \uu(\om) \phi(\om,y)dP(\om): \uu \in \Lp{P}{2}\right\}, \]
with the norm \[\lV h \rVH = \inf\left\{ \lV \uu \rVL{P}{2} : \uu\in S^{-1}(h)\right\}~.\]
\end{corollary}

In Appendix \ref{app:proofThm22}, we provide an elementary proof 
of Corollary \ref{cor:integral_rkhs} based on \Cref{integralrkhs},
which differs from the proof in \cite{integral_rkhs} based on the direct integral theory \cite{nielsen2020direct}.
As in \cite{integral_rkhs}, our proof of Corollary \ref{cor:integral_rkhs} is still valid when $\X$ is an arbitrary set.

\subsection{Lipschitz constant of feature map with differentiable integral kernels}
\label{sec:diffIntKernel}

We investigate sufficient conditions on an integral kernel $k$ so that its feature map $\vphi$ is Lipschitz-continuous. 
For this purpose, we introduce a 
differentiability assumption on these kernels where $\Lip(\vphi)$ can be obtained as a function of the norm of the Fréchet derivative of $\vphi$ on $\X$.

In the literature, a common way to analyze the Lipschitz-continuous 
of the $\vphi$ is assume that $\phi\in\LpX{P}{2}$ and that there exists $F\in\RR$ such that $\Lip(\phi(\om,\cdot))\leq F$, $P$-a.e. As a consequence, $\vphi$ is Lipschitz-continuous with upper bound $F$ (\cite{fiedler_lipschitz_2023} Theorem 5.5, \cite{D2KE} Appendix 1.1). The next proposition gives a tighter bound for $\phi$ that is not necessarily uniformly Lipschitz. This proposition is based on an  assumption which is similar to the one in  our main result of this section in Theorem \ref{2 inf lip}.

\begin{proposition}
\label{prop:upperbound}
    Choose $\phi\in\LpX{P}{2}$ such that there exists $F\in\Lp{P}{2}$ with $P$-a.e. $\Lip(\phi(\om,\cdot)) \leq F(\om)$, then $
    \Lip(\vphi) \leq \lV F\rVL{P}{2}$.
\end{proposition}

This result can be obtained by substituting their uniform Lipschitz constant by the function $F$, in the proof of Theorem~5.5 in \cite{fiedler_lipschitz_2023}. It is valid in every metric space $\X$. To establish an exact formula for $\Lip(\vphi)$, we further assume that $\X$ is a subset of Banach space in order to analyze the derivative of $\vphi$. In this case, the Lipschitz constant can be expressed as the supremum of a suitably defined derivative, rendering it tractable in analysis.

\begin{theorem}
\label{2 inf lip}
    Let $\X$ be an open and convex subset of a Banach space $E$. Consider  $\phi\in\LpX{P}{2}$
    with an integral kernel $k$ and assume further that :
    \begin{enumerate}[(i)]

    \item There exists $F\in \Lp{P}{2}$ such that for $P$-almost all $\om\in\Om$, $\Lip(\phi(\om,\cdot)) \leq F(\om)$.
    \item For all $x \in\X$ and $P$-almost all $\om\in\Om$, $ \phi(\om,\cdot)$ is Fréchet differentiable at $x$.

    \item For all $x$ and $z$ in $\X$, $D_x\phi(\cdot,x)[z]$ is measurable.\footnote{We assume that $D_x \phi(\om, x)=0$ when $\phi(\om,x)$ is not differentiable at $x$.}

    \end{enumerate}
    
     Then, the associated feature map $\vphi$ is Fréchet differentiable everywhere on $\X$ and it is Lipschitz-continuous with \begin{align*}
         Lip(\vphi) &= \sup_{x\in\X,\z\in \BE}\lV D_x\phi(\om,x)[\z] \rVL{P}{2}\\
         &= \sup_{x'\in\X,z\in \BE}D_x D_y k(x',x')[z,z]^{\frac{1}{2}}~.
         \end{align*}
\end{theorem}
The proof can be found in Appendix \ref{app:proofThm25}. 
It is based on representing the Fréchet derivative of 
$\vphi(x)$ as the projection of $\phi(\cdot,x)$ using the operator $S$, 
in the sense that for all $\z\in E$, $D_x\vphi(x)[\z] = S(D_x\phi( \cdot,x)[\z])$.
Then the problem to analyze the operator norm of $D_x\vphi(x)$ 
(to obtain the Lipschitz constant of $\vphi$)
is reduced to analyzing the norm of $ D_x\phi( \cdot,x)[\z]$ on $\Lp{P}{2}$.
The second expression with the second-order derivative $D_x D_y k(x,y): E \to ( E \to \R)$ evaluated at $x=x',y=x'$
is actually more general. 
It can be derived directly from the hypothesis that 
$\vphi$ is Fréchet differentiable on $\X$ (which is a consequence of the assumptions in \Cref{2 inf lip}). 
This expression exposes how the Lipschitz regularity of $\vphi$ depends on the highest curvature on the diagonal of $k$.

In the next section, we shall see that the conditions of Theorem \ref{2 inf lip} hold for infinite random neural networks with the ReLU activation function. More generally, it holds with a Lipschitz-continuous activation function and with an integrability condition.
The last result of this section is a sufficient condition for the feature map to be not Lipschitz continuous. It is based on hypothesis similar to \cref{2 inf lip}.
\begin{theorem}
\label{2 inf}
    Let $\X$ be an open and convex subset of a Banach space $E$. Consider  $\phi\in\LpX{P}{2}$
    with an integral kernel $k$ and assume further that :
    \begin{enumerate}[(i)]

    \item For all $x \in\X$ and $P$-almost all $\om\in\Om$, $ \phi(\om,\cdot)$ is Fréchet differentiable at $x$.

    \item For all $x$ and $z$ in $\X$, $D_x\phi(\cdot,x)[z]$ is measurable.

    \end{enumerate}

    Then, 
\[\sup_{x\in\X,\z\in \BE}\lV D_x\phi(\om,x)[\z] \rVL{P}{2}=+\infty\implies\Lip(\vphi) =+\infty~.\]    
\end{theorem}
The proof is given in Appendix \ref{app:proofThm26}.
 In the following example, we show how \cref{2 inf} can be applied to the Wiener kernel. 
 \begin{exmp} (Wiener)
     Let $\X = (0,1)$, and $k(x,y) = min(x,y)$ for all $x,y \in \X$. The mercer decomposition of $k$ is
     \begin{align*}
     k(x,y) &= \sum_{n=1}^\infty \phi_n(x)\phi_n(y) \lambda_n\\
     &= \int_{\mathbb{N}^*}\phi_n(x)\phi_n(y) dP(n) ~,
       \end{align*}
       where $\phi_n(x) = \sqrt{2}\sin\left((n-\frac{1}{2})\pi x\right)$, $\lambda_n = (n-\frac{1}{2})^{-2}\pi^{-2}$ and $P = \sum_{i=1}^\infty\lambda_n\delta_n$. We have
       \begin{align*}
           \lV \phi_n(0)'\rV_{l^2(P)}&= \sum_{n=1}^\infty 2\cos^2\left(0\right)=\infty ~.
       \end{align*}
       Applying \cref{2 inf}, we deduce that $\vphi$ is not Lipschitz continuous.
 \end{exmp}

\section{Application to infinite random neural networks}
\label{sec:appRNN}

We study the Lipschitz upper bound of a family of infinite random neural networks, that is, the functions in the RKHS of an integral kernel of a neuron. By a neuron, we mean a function of the form
\[\phi(\om,x) = \sigma(w^\T  x+b)~,\]
where $\om = (w,b)$ in $\Om = \RR^d\times\RR$, $x$ in an open and convex $\X \subset \RR^d$ and $\sigma:\RR\rightarrow\RR$. Let $P$ be a probability measure on $(\Om,\mathcal{B}(\RR^{d+1}))$, with $\mathcal{B}(\RR^{d+1})$ being the Borel $\sigma$-field. We can represent an infinite random neural network $h$ by 
Theorem \ref{integralrkhs},
\[h(x) = \int_\Om u(w,b) \sigma(w^\T  x+b)dP(w,b)~,
\]
where $u\in\Lp{P}{2}$.
It can be seen as a generalization of neural networks
with a functional parameter $u$. 
Such an integral representation of $h$ is similar to the neural network in the mean-field regime \cite{meanfield}, except that here the distribution on $(w,b)$ is fixed and we learn only the function $u$. The following corollary asserts that when the activation function $\sigma$ is Lipschitz-continuous, the infinite random neural network $h$ is also Lipschitz-continuous under a suitable integrable condition. 

\begin{corollary}
\label{cor:lip_NN}
    Suppose that $\sigma$ is  a Lipschitz-continuous function and that $(w,b)\mapsto\lV  w\rVe + |b|\in\Lp{P}{2}$. Then, $\phi\in\LpX{P}{2}$ and $h$ is Lipschitz-continuous with $\Lip(h)\leq \Lip(\sigma)\sqrt{\EP \lV w \rVe^2 \EP u(w,b)^2}$. 
\end{corollary} 

The proof of Corollary \ref{cor:lip_NN} is provided in Appendix \ref{app:lip_NN}. This corollary gives basic conditions on $\sigma$ and $P$ where $\phi\in\LpX{P}{2}$ and provides an upper bound a Lipschitz upper bound on infinite random neural networks. However, this bound relies on the upper bound of $Lip(\vphi)$ 
given in \Cref{prop:upperbound} which 
does not take into account of the coupling between the $\sigma$ and weights $(w,b)$. We next improve this bound
in two cases where \Cref{2 inf lip} can be applied.

To be consistent with the assumption (iii) in Theorem \ref{2 inf lip}, we define $\sigma'(u) = 0$ on all non-differentiable points $u$ of $\sigma$. By the Rademacher's Theorem (see \cite[Theorem 3.1.6]{federer1969geometric}), a Lipschitz-continuous $\sigma$ is differentiable almost everywhere.

\begin{proposition}
    \label{prop:RNNlipinf}
     With the assumption of Corollary \ref{cor:lip_NN}, if one of the following assertions is true
     \begin{enumerate}[(i)]
         \item  $\sigma$ is differentiable in $\RR$, and $\sigma'$ is measurable.
         \item  $P$ is absolutely continuous and $\sigma'$ is measurable,
     \end{enumerate}
     then the assumptions of Theorem \ref{2 inf lip} hold.
\end{proposition}

The proof is given in Appendix \ref{app:lip_NN2}. Now that we have a framework for applying Theorem \ref{2 inf lip}, we will present a case in which computing the Lipschitz constant is straightforward. This case involves a neural network whose weights are drawn from isotropic Gaussian distributions.

Let $\mu_\gamma^d$ be the Gaussian measure in $\RR^d$, that is, for all $A\in\mathcal{B}(\RR^{d})$, \[\mu_\gamma^d(A) = \frac{1}{\gamma^d\sqrt{2\pi }^d}\int_A \text{exp}\left(-\frac{1}{2\gamma^2}\lV w \rVe^2 \right)d\lambda^d(w)~,\] where $\lambda^d$ is the Lebesgue measure on $\RR^d$.

The next result shows that we can compute the exact Lipschitz bound as the supremum over a two-dimensional integral. Its proof is deferred to \Cref{app:lip_NN3}.
\begin{theorem}[Infinite random neural network with Gaussian distribution]
\label{propRNN}
     Let $\sigma$ be a Lipschitz-continuous activation function such that $\sigma'$ is measurable. Define $P$ as a 
     product probability $\mu_\gamma^d\otimes p_b$ where $\gamma>0$ and $p_b$ is an absolutely continuous measure such that $\mathbb{E}_{p_b} b^2<\infty$. Then $\phi\in\LpX{P}{2}$, and we have for every $x\in\X$ such that $x\neq 0$ and $\z\in\BR$,
    \begin{align*}
    \EP\left[ (D_x\sigma (w^\T x+b ) [\z])^2\right]&=\frac{(x^\T \z )^2}{\lV x \rVe^2}\mathbb{E}_{\zeta,b \sim \mu_\gamma^1\otimes p_b} \left[(\zeta^2-1)\sigma'(\zeta \lV x \rVe + b)^2 \right]+\mathbb{E}_{\zeta,b \sim \mu_\gamma^1\otimes p_b}\left[ \sigma'(\zeta \lV x \rVe + b)^2 \right].
    \end{align*}    
Furthermore, the feature map $\vphi$ of the integral kernel of $\phi$ satisfies
\[
    \Lip(\vphi) =\sup_{x\in\X}\sqrt{\mathbb{E}_{\zeta,b \sim \mu_\gamma^1\otimes p_b}[ \zeta^2\sigma'(\zeta \lV x \rVe + b)^2 ]}~.
\]
\end{theorem}

Next, we expose two important examples of infinite random neural networks, the random Fourier features and the ReLU random neural network. In both cases, $\X$, $\Om$, and $p_w$ are defined as above.

\begin{exmp}[Random Fourier Feature (RFF)]
\label{GRFF}
Let $p_b$ be the uniform measure on $[0,2\pi]$. Consider that $\phi_\text{RFF}$ is defined with $\sigma = \sqrt{2}\cos$. Then, the integral kernel of $\phi_\text{RFF}$ is the Gaussian kernel $k_\text{RFF}$ (\cite{rahimi_random_nodate}) defined as
\begin{align*}
        k_\text{RFF}:\RR^d\times\RR^d &\rightarrow \RR\\
        (x,x') &\mapsto e^{-\frac{\gamma^2}{2}\lV x-x'\rVe^2}~.
\end{align*}
Let $\varphi_\text{RFF}$ be the features maps associated with $k_\text{RFF}$.
We obtain from \Cref{propRNN} that \[\Lip(\varphi_\text{RFF}) = \gamma~.\] This result is a well known fact for Gaussian kernels. A more detailed analysis on random Fourier Features is given in Section \ref{sec:RFF}.
 \end{exmp}

\begin{exmp}[ReLU Random Neural Network]
\label{relu}
Let $p_b$ be an arbitrarily symmetric probability measure and $\phi_\text{ReLU}$ be defined with the ReLU activation function $\sigma : x\mapsto \max(x,0)$.
 The integral kernel $k_\text{ReLU}$ of $\phi_\text{ReLU}$ is known as the \textit{neural network Gaussian process kernel} \cite{nnGP}. By \Cref{propRNN} we get \[\Lip(\varphi_\text{ReLU}) = \frac{ \gamma}{\sqrt{2}}~.\] If we consider the classical Lipschitz bound for multi-layer perceptron as in \cite[Proposition 1]{virmaux_lipschitz_2018}, it leads to the upper bound $\gamma$. 
\end{exmp}

We remark that in both of the examples, the choice of $\X$ does not matter in the above computations. 
In fact, to compute $\Lip(\varphi_\text{RFF})$, we use the property that the phase $b$ is uniformly random on $[0, 2 \pi]$, and therefore $\mathbb{E}_{\zeta,b \sim \mu_\gamma^1\otimes p_b}[ \zeta^2\sigma'(\zeta \lV x \rVe + b)^2 ] = 
 2\mathbb{E}_{\zeta,b \sim \mu_\gamma^1\otimes p_b}[ \zeta^2 \sin^2(b) ] = \gamma^2, \forall x \in \X$.
 To obtain $\Lip(\varphi_\text{ReLU})$, we use the symmetry property of $\mu_\gamma^1 $ and $p_b$ to compute
 \begin{align*}
    \mathbb{E}_{\zeta,b \sim \mu_\gamma^1\otimes p_b}[ \zeta^2\sigma'(\zeta \lV x \rVe + b)^2 ] 
    &= \mathbb{E}_{\zeta,b \sim \mu_\gamma^1\otimes p_b}[ \zeta^2   \mathbbm{1}_{ \zeta \lV x \rVe + b > 0 }  ] \\
    &= \int_{0}^\infty \zeta^2   \left ( \int_{ - \zeta \lV x \rVe }^\infty   p_b (u) du +   \int_{  \zeta \lV x \rVe }^\infty   p_b (u) du     \right ) \mu_\gamma^1 (d \zeta )  = \frac{\gamma^2}{2}. 
 \end{align*}
 Therefore, the formula in \Cref{GRFF} and \Cref{relu} is valid for any non-empty, open and convex $\X \subset \R^d$.

\section{Application to continuous and shift-invariant kernels} \label{sec:RFF}
Let $k$ be a continuous, positive definite, and shift-invariant (i.e. there exists $\kappa$ such that for every $x,y\in\X$, $k(x,y) = \kappa(x-y)\in\RR$) kernel on an open and convex subset $\X\subset\RR^d$. According to \cite[Theorem 1.4.3]{rudin1962fourier}, there exists a non-negative finite measure $\mu$ such that for all $x,y\in \X$, \[
k(x,y) = \int_{\RR^d} e^{jw^\T (x-y)} d\mu(w) ~.
\]
We assume throughout this section $\mu(\RR^d) = \int_{\RR^d} d\mu(w) = \kappa(0) \in (0, \infty)$. Hence, $p_w = \frac{\mu}{\kappa(0)}$ is a probability measure on $\RR^d$. In the following, we call $\mu$ and $p_w$ respectively the Fourrier transform and the normalized Fourier transform of $\kappa$. Thus, we have
\begin{align*}
    k(x,y) &= \E_{p_w} \kappa(0)e^{jw^\T (x-y)}\\
    &= \E_{p_w} \kappa(0)\cos\left(w^\T (x-y)\right)+ j\sin\left(w^\T (x-y)\right)\\
    &= \E_{p_w} \kappa(0)\cos\left(w^\T (x-y)\right),
\end{align*}
 where $p_w$ is the normalized Fourier transform of $\kappa$. The term with sinus vanished since $k$ is real ($p_w (A) = p_w (-A)$ for any measurable $A \subset \R^d$).

 In order to write this kernel as an integral kernel, we will use the same trick as in \cite{rahimi_random_nodate}. Note that for all $a,c\in\RR$, \begin{align*}
 \E_{p_b} \cos(a+b) \cos(c+b) &=  \E_{p_b} \frac{1}{2} \left(\cos(a-c) + \cos(a+c+2b)\right) 
 = \frac{1}{2} \cos(a-c) ,
  \end{align*}
where $p_b$ is the uniform law on $[0,2\pi]$.
Finally, we have \[k(x,y) = \E_{w\sim p_w, b\sim p_b} \sqrt{2\kappa(0)}\cos(w^\T x +b)\sqrt{2\kappa(0)}\cos(w^\T y +b).\]
In \cite[Example C.2]{shift_inv}, it is shown that this kind of kernel is such that $\vphi$ is $\lambda_{\max} \left(\kappa(0)\E_{p_w} w w^\T\right)^\frac{1}{2}$-Lipschitz continuous when $\E_{p_w} \lV w \rV^2$ is finite. In the following theorem, we show that $\E_{p_w} \lV w \rV^2<\infty$ is also a necessary condition and that $\lambda_{\max} \left(\kappa(0)\E_{p_w} w w^\T\right)^\frac{1}{2}$ is the Lipschitz constant.
\begin{theorem}
\label{THMRFF}
Let $k(x,y)=\kappa(x-y)$ be a continuous, shift-invariant, positive definite kernel, and let $p_w$ denote the normalized Fourier transform of $\kappa$.

\noindent
If $\E_{p_w}\|w\|^2<\infty$, then the feature map $\vphi$ is Lipschitz continuous, $\kappa$ is twice differentiable on $0$, and
\[
\Lip(\vphi)
= \sqrt{\kappa(0)\,\lambda_{\max}\!\bigl(\Cov(w)\bigr)}
= \sqrt{\lambda_{\max}\!\bigl(-\nabla^2 \kappa(0)\bigr)}.
\]
Otherwise, $\Lip(\vphi)=+\infty$.

\noindent 
Here, $\Cov(w)=\E_{w\sim p_w}[ww^\top]$ and $\lambda_{\max}(\cdot)$ denotes the largest eigenvalue.
\end{theorem}

The proof is given in Appendix \ref{thmRFF}.

\begin{exmp} (Gaussian)
\label{ani_gauss}
Consider the anisotropic Gaussian kernel defined by
\[
k(x,y)=\kappa(x-y), \qquad x,y\in\X,
\]
where, for all $\Delta\in\RR^d$,
\[
\kappa(\Delta)=\exp\!\left(-\tfrac12\,\Delta^\T\Sigma\Delta\right),
\]
and $\Sigma$ is a symmetric positive definite matrix.

The Fourier transform of $\kappa$ is the multivariate Gaussian distribution
$\mathcal{N}(0,\Sigma)$ \cite{rff_classic}. Consequently, the associated normalized
Fourier measure $p_w$ satisfies $\Cov(w)=\Sigma$. Applying \Cref{THMRFF} yields
\[
\Lip(\vphi)=\sqrt{\lambda_{\max}(\Sigma)}.
\]
\end{exmp}

\begin{exmp} (Matèrn)
\label{mattern}
Let $\nu>0$ and let $\Sigma$ be a symmetric positive definite matrix. The anisotropic
Matérn kernel is defined by
\begin{align*}
k(x,y) &= \kappa(x-y), \\
\kappa(\Delta) &=
\frac{2^{1-\nu}}{\Gamma(\nu)}
\left(\sqrt{2\nu}\,\Delta^\T\Sigma\Delta\right)^\nu
J_\nu\!\left(\sqrt{2\nu}\,\Delta^\T\Sigma\Delta\right),
\end{align*}
for $x,y\in\X$ and $\Delta\in\RR^d$, where $J_\nu$ denotes the modified Bessel function of the second kind \cite{rff_classic}.
Moreover, one has $\kappa(0)=1$. Indeed, letting
$z=\sqrt{2\nu}\,\Delta^\top\Sigma\Delta$, we can write
\[
\kappa(\Delta)
=
\frac{2^{1-\nu}}{\Gamma(\nu)}\, z^\nu J_\nu(z).
\]
Using the asymptotic behavior
$z^\nu J_\nu(z)\to 2^{\nu-1}\Gamma(\nu)$ as $z\to 0$
(see \cite[Eq.~10.30.2]{dlmf}),
we obtain $\kappa(\Delta)\to 1$ as $\Delta\to 0$.
According to \cite{rff_classic}, the Fourier transform of $\kappa$ can be written as

\[
\mu(w)
=
\left(1+\frac{w^\T\Sigma w}{2\nu}\right)^{-\frac{d+2\nu}{2}}.
\]
Actually, once normalized, it is the density of the multivariate
Student distribution $t_{2\nu}(0,\Sigma^{-1})$. 
When $\nu>1$, it is known \cite{student} that
$w\sim t_{2\nu}(0,\Sigma^{-1})$ satisfies
$\E_{p_w}\lV w\rVe^2<\infty$, with covariance matrix
\[
\Cov(w)=\frac{2\nu}{2\nu-2}\,\Sigma^{-1}.
\]
Applying \Cref{THMRFF}, we obtain 
\[
\Lip(\vphi)=\sqrt{\frac{\nu}{\nu-1}\,\lambda_{\max}(\Sigma^{-1})}.
\]
On the other hand, if $\nu\le1$, the second-order moment of $p_w$ is infinite and therefore $\vphi$ is not Lipschitz continuous. To see this explicitly, consider the isotropic case $\Sigma=I_d$. Using polar coordinates,
we may write
\[
\E_{p_w}\lV w\rV^2
\propto
\int_0^{\infty}
r^2\left(1+\frac{r^2}{2\nu}\right)^{-\frac{d+2\nu}{2}}r^{d-1}\,dr.
\]
As $r\to\infty$, the integrand is equivalent to
$r\mapsto 2\nu^{\frac{d+2\nu}{2}}\,r^{1-2\nu}$, which is integrable if and only if $\nu>1$.
\end{exmp}

\begin{exmp} (Laplace)
The isotropic Laplace kernel is defined by
\[
k(x,y)=\kappa(x-y)=\exp\!\bigl(-\lV x-y\rVe\bigr),
\qquad x,y\in\X.
\]
Its Fourier transform is proportional to the density of a Cauchy distribution
\cite{rff_classic}:

\[
\mu(w)
=
\left(1+\lV w\rV^2\right)^{-\frac{d+1}{2}}.
\]

It is well known that the Cauchy distribution does not admit a finite second-order moment.
This can also be recovered as a special case of the Matérn family with parameter
$\nu=\tfrac12$. Consequently, the associated feature map $\vphi$ is not Lipschitz
continuous.
\end{exmp}

\section{Numerical illustration and an open question on finite random features}
\label{sec:num}

This section has two objectives. First, we provide a numerical illustration of the main results of this article, which establish exact Lipschitz constants for several integral kernel feature maps. Second, we empirically investigate an open question concerning the asymptotic behavior of random feature approximations when the integral kernel in \Cref{integral_kernel} is computed from finite random features.
When there are a large number of random features, \cite{iaaf009} provides an upper and lower bound 
of $\Lip(\ha)$ when the weights of each layer of $\ha$ are random. 
In our case, we do not assume that the weights of the last layer is random 
so that it can be optimized as in kernel ridge regression. 

\begin{definition}[Random features map]
\label{def:rff}
Let $\phi\in L^2(\Omega\times\X,P)$ and let $(\omega_i)_{i\geq1}$ be an i.i.d.\ sequence with distribution $P$. For any $N\in\mathbb{N}$, we define the random feature map $\theta_N:\X\to\mathbb{R}^N$ by
\[
\theta_N(x)=\frac{1}{\sqrt{N}}
\begin{bmatrix}
\phi(\omega_1,x)\\
\phi(\omega_2,x)\\
\vdots\\
\phi(\omega_N,x)
\end{bmatrix}.
\]
\end{definition}

The associated empirical kernel $k_N(x,x')=\theta_N(x)^\top\theta_N(x')$ approximates the integral kernel $k$ induced by $\phi$. By the law of large numbers, $k_N(x,x')\to k(x,x')$ almost surely for all $x,x'\in\X$, and the convergence is uniform on compact sets under mild assumptions \cite{rahimi_random_nodate}. This approximation plays a central role in scalable kernel methods, where kernel ridge regression with $k_N$ reduces to linear regression in $\mathbb{R}^N$. 

While convergence properties of $k_N$ are well understood, much less is known about the robustness properties of random feature maps. In particular, it is natural to ask whether the Lipschitz constant of $\theta_N$ converges to that of the infinite-dimensional feature map $\vphi:\X\to\HH$, defined by $\vphi(x)=k(\cdot,x)$, i.e. 

\medskip
\noindent\textbf{Open question.}
Does $\Lip(\theta_N)$ converge to $\Lip(\vphi)$, in probability or almost surely, as $N\to\infty$?

\medskip
To explore this question numerically, we adopt a quantile-based formulation. Fix $\delta\in(0,1)$ (here $\delta=0.9$), and define $t_N$ such that
\[
\delta=\mathbb{P}\!\left( \Lip(\theta_N)\leq \Lip(\vphi)+t_N \right).
\]
Based on the preliminary results in \cite{reverdi:tel-05158402} using empirical process theory, we expect that $t_N\to0$ as $N\to\infty$.

We consider $\X=(-1,1)$ and estimate $t_N$ empirically using Monte Carlo simulations. For each value of $N$, we generate $I=3000$ independent realizations of the random feature map $\theta_N^{(i)}$. For each realization, we compute an empirical estimate of the Lipschitz constant,
\[
\widehat{\Lip}(\theta_N^{(i)})=\max_{j=1,\dots,99} \left\| (\theta_N^{(i)})'(x_j) \right\|,
\qquad x_j=-1+\tfrac{2j}{100},
\]
which provides a reliable approximation due to the regularity of the functions involved. The empirical distribution  estimated from $\{\widehat{\Lip}(\theta_N^{(i)})\}_{i \leq N}$ allows us to compute an empirical quantile $\widehat{t}_N$.

We perform this experiment for three representative integral kernels studied in this paper: random Fourier features of the Gaussian kernel, random Fourier features of the Matérn kernel, and one-hidden-layer ReLU networks. Figures~\ref{fig:Liprff}, \ref{fig:Liprelu}, and \ref{fig:Lipmattern} show the evolution of $\widehat{t}_N$ as a function of $N$. In all cases, we observe a clear convergence towards $0$. The exact Lipschitz constant $\Lip(\vphi)$ derived in the previous sections is used to compute $\widehat{t}_N$, providing strong numerical evidence in support of the conjectured convergence of $\Lip(\theta_N)$.
\begin{figure}[ht]
    \centering
    \begin{subfigure}{0.45\textwidth}
        \centering
        \includegraphics[width=\textwidth]{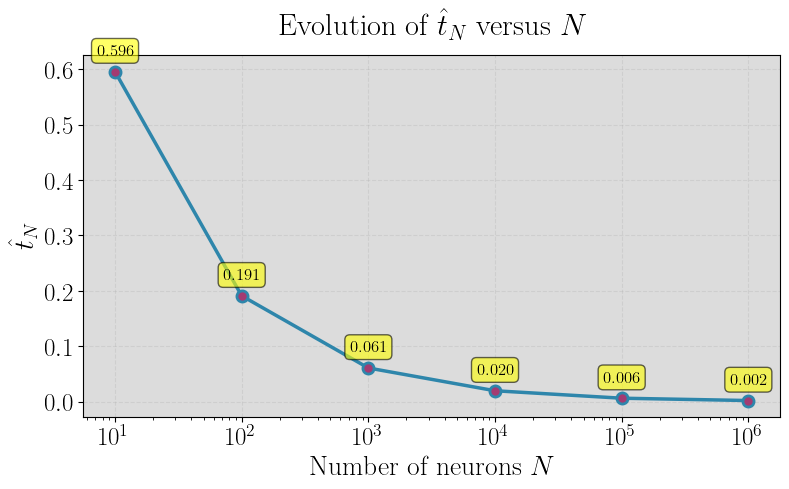}
        \caption{Random Fourier features of Gaussian kernel.}
        \label{fig:Liprff}
    \end{subfigure}
    \hfill
    \begin{subfigure}{0.45\textwidth}
        \centering
        \includegraphics[width=\textwidth]{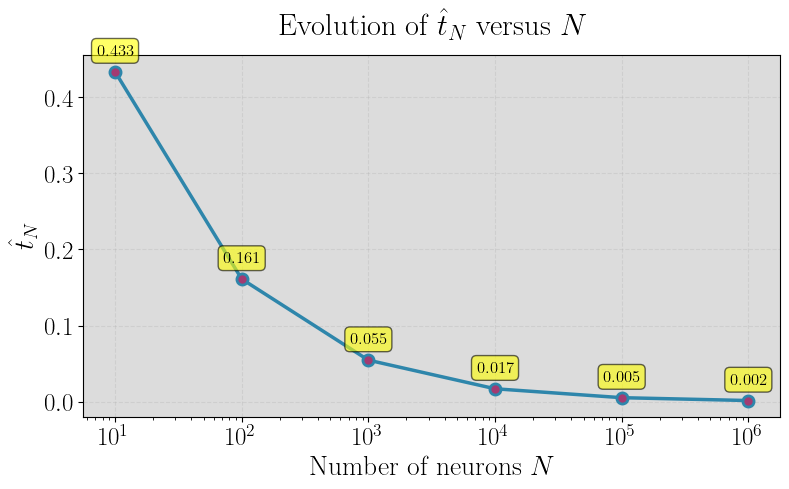}
        \caption{ReLU neural networks.}
        \label{fig:Liprelu}
    \end{subfigure}

    \medskip

    \makebox[\textwidth][c]{%
        \begin{subfigure}{0.45\textwidth}
            \centering
            \includegraphics[width=\textwidth]{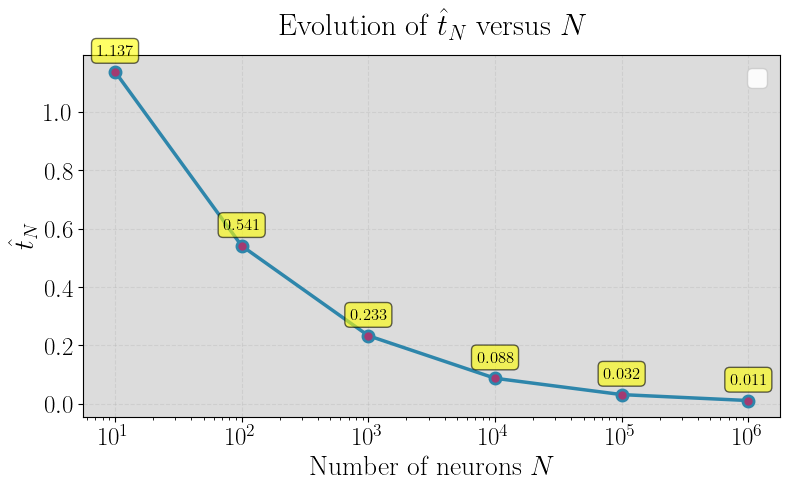}
            \caption{Random Fourier features of Matérn kernel.}
            \label{fig:Lipmattern}
        \end{subfigure}
    }

    \caption{Convergence of the estimated Lipschitz quantile $\hat{t}_N$ towards $0$ as $N$ grows on \Cref{GRFF}, \Cref{relu} and \Cref{mattern}.}
    \label{fig:Lip}
\end{figure}

\section{Conclusion}
\label{sec:con}
In this paper, we investigated the regularity properties of the feature map associated with integral kernels under differentiability assumptions. We established sufficient conditions ensuring the Lipschitz continuity of feature map and derived explicit expressions for the corresponding Lipschitz constants. Within the general framework of integral kernels, we also identified a sufficient condition under which the feature map fails to be Lipschitz continuous.
Based on these results, we obtain analytical formula for the Lipschitz constants for two widely used kernel families: 1). infinite neural network kernels associated with isotropic Gaussian distributions. 2). continuous and shift-invariant kernels. In the first case, we showed that the Lipschitz constant can be expressed as the supremum of a two-dimensional integral, which allows us to identify the Lipschitz constants for Gaussian kernel and ReLU random neural network. In the second case, we proved that the associated feature map is Lipschitz continuous if and only if the weight distribution admits a finite second-order moment. Moreover, we related the Lipschitz constant to the largest eigenvalue of the covariance matrix of the weights, or equivalently to the Hessian of the kernel at the origin.

Finally, we studied numerically an open question concerning the asymptotic behavior of random feature approximations, namely the convergence of the Lipschitz constant of a random feature map $\theta_N$ towards that of the corresponding integral feature map. We provided three numerical examples supporting this behavior, showing that a suitable quantile of the difference between the two Lipschitz constants converges to zero as the number of random features increases. This indicates that when the $\Lip(\vphi)$ is analytically computable, we could use this asymptotic result to control $\Lip(\theta_N)$. This is particularly desirable in high-dimension since $\Lip(\theta_N)$ can be NP-hard to compute \cite[Theorem 2]{virmaux_lipschitz_2018}. On the other hand, 
computing numerically the Lipschitz constant of an integral feature map $\vphi$ is in general challenging, since $\vphi$ takes values in an infinite-dimensional Hilbert space. In this case, one could hope to use the computational tools \cite{fazlyab_efficient_2019,latorreLipschitzConstantEstimation2019} for $\Lip(\theta_N)$ to obtain a tight Lipschitz bound of $\vphi$ .

\appendix
\counterwithin{equation}{subsection}
\renewcommand{\theequation}{\Alph{section}.\arabic{subsection}.\arabic{equation}}

\counterwithin{theorem}{subsection}

\section{Notation}
\begin{table}[h!]
\small
\centering
\begin{tabular}{lll}
\toprule
\textbf{Category} & \textbf{Symbol} & \textbf{Description} \\ \midrule
\multicolumn{3}{l}{\textbf{General Notations}} \\
 & $\X$ & Input space \\
 & $d$ & Input dimension when finite \\
 & $N$ & Number of neurons \\

\multicolumn{3}{l}{\textbf{Kernel Methods}} \\
 & $k$ & Kernel function \\
 & $\HH$ & RKHS associated with $k$ \\
 & $\vphi : \X \mapsto \HH$ & Feature map \\

\multicolumn{3}{l}{\textbf{Linear Algebra}} \\
 & $\Span$ & Linear span \\
 & $\lV \cdot \rVe$ & Euclidian norm on $\mathbb{R}^d$ \\
 & $\lb x,y \rb = x^\T y $ & Canonical inner product on $\mathbb{R}^d$ with $x,y\in\RR^d$ \\
 & $A^\T$ & The transpose matrix of $A$ \\

\multicolumn{3}{l}{\textbf{Banach Spaces $E$ and $F$}} \\
 & $\lV \cdot \rV_E$ & Norm in $E$ \\
 & $\BE$ & Closed unit ball of $E$  \\
 & $\lV T \rVop = \sup_{\z \in \BE, z \neq 0} \frac{\lV T\z \rV_F}{\lV \z \rV_E}$ & Operator norm of a linear operator $T : E \mapsto F$  \\
 & $Df$ & Fréchet derivative of a function $f:E\mapsto F$ \\
 & $\Lip(f)$ & Lipschitz constant or optimal Lipschitz bound \\
\multicolumn{3}{l}{\textbf{Inner (or Semi-inner) Product Space $H$}} \\
 & $\lV \cdot \rV_H$ & Norm (resp. semi-norm) in $H$ \\
 & $\lb \cdot , \cdot \rb_H$ & Inner (resp. semi-inner) product in $H$ \\

\multicolumn{3}{l}{\textbf{Measure Theory}} \\
 & $(\Om,\mathcal{A},P)$ & Measure space \\
 & $\lambda^d$ & Lebesgue measure on $\RR^d$\\
 & $\mu_\gamma^d$ & Multivariate Gaussian measure on $\RR^d$\\
 & $\EP u = \int_\Om u \, dP$ & Expectation of a measurable $u:\Om \mapsto \mathbb{R}$ \\

 & $\mathcal{B}(T)$ & Borel $\sigma$-field of topological space $T$ \\
 & $\LLp{P}{2}$
 & Semi-inner product space of square integrable functions \\
 & $\Lp{P}{2}$ & Hilbert space of equivalence classes of $\LLp{P}{2}$ \\
   & $\LpX{P}{2}$   & Set of functions $\phi:\Om\times\X\rightarrow\RR$ s.t. $\forall x\in\X,\phi(\cdot,x)\in\Lp{P}{2}$ \\
 & $\lV u \rV_{\Lp{P}{p}} = \left(\int_\Om |v|^p dP\right)^{1/p}$ & $p$-norm in $\Lp{P}{p}$, $v$ a representative of $u$ and $1 \leq p\leq\infty$ \\

\bottomrule
\end{tabular}
\caption{Notation Table}
\end{table}

\section{Background}

\subsection{$\LLp{P}{2}$ and $\Lp{P}{2}$ spaces}
\label{Appendix:Lp}
Let $(\Omega,\mathcal{A},P)$ be a measure space. We consider the set $\LLp{P}{2}$ of real-valued functions on $\Omega$ that are square-integrable with respect to $P$. For $\uu,\uv \in \LLp{P}{2}$ and $\lambda \in \mathbb{R}$, we define pointwise addition and scalar multiplication as $(\uu+\lambda \uv)(x) = \uu(x)+\lambda \uv(x)$. Then $\LLp{P}{2}$ is a vector space. The mapping
\begin{equation}
\lb \uu,\uv \rbLL{P}{2} = \int_\Omega \uu(\omega)\uv(\omega)\,dP(\omega)~,
\end{equation}
is a semi-inner product on $\LLp{P}{2}$. We then introduce the equivalence relation $\uu \sim \uv \iff \|\uu-\uv\|_{L^2(P)}=0$ and consider the quotient space $\Lp{P}{2} := \LLp{P}{2}/\!\sim$. With the induced inner product (an abuse of notation)
\begin{equation}
\lb \uu,\uv\rbL{P}{2}:= \lb [\uu],[\uv]\rbL{P}{2} = \int_\Omega \uu(\omega)\uv(\omega)\,dP(\omega)~,
\end{equation}
the space $(\Lp{P}{2}, \lb\cdot,\cdot\rbL{P}{2})$ is a Hilbert space. Note that the inner product between two equivalence classes $[\uu]$ and $[\uv]$ coincides with the integral of the product of their representatives. 

\subsection{Kernel methods and Reproducing Kernel Hilbert Space (RKHS)}
\label{Appendix:kernel}

In this section, we present some basic definitions and properties of RKHS. First, consider an arbitrary input space $\X$.
\begin{definition}
    A function $k:\X\times \X\rightarrow\mathbb{R}$ is positive definite and symmetric (PDS) if for any $m\in\mathbb{N}$ and $\{x_1,\ldots,x_m\}\subset \X$, the matrix $\left(k(x_i,x_j)\right)_{1\leq i,j\leq m} $ is positive semi-definite and symmetric.
\end{definition}

\begin{definition}[\cite{rkhs_base}]
    An RKHS $\mathcal{H}$ on $\X$ is a Hilbert space of functions in $\mathbb{R}^{\X}$ where the evaluation functions are continuous. That is, for every $x\in\X$, $h\mapsto h(x)$ is continuous for the topology induced by the inner product of $\mathcal{H}$. 
\end{definition}
\begin{theorem}[Aronszajn-Moore Theorem and Uniqueness of the RKHS \cite{rkhs_base}]
\label{moore}
    A PDS kernel induces a unique RKHS $\HH$ and the feature map $\vphi:x\mapsto k(x,\cdot)$ such that the reproducing property holds, that is, for every $x\in\X$ and $h\in\HH$, $f(x) = \lb h,\vphi(x)\rb_\HH$ where $\lb\cdot,\cdot\rb_\HH$ is the inner product of $\HH$.  
    \end{theorem}

\subsection{Banach Analysis}

\label{Appendix:Banach}
\begin{definition}[Lipschitz continuity]
Let $(A,d_A)$ and $(B,d_B)$ be two metric spaces. Let $f:A\rightarrow B$. We say that $f$ is Lipschitz-continuous with bound $C>0$ if for all $ a,a'\in A$,\[d_B(f(a),f(a'))\leq Cd_A(a,a')~.\]
Moreover, we define the optimal Lipschitz bound by $\Lip(f) = \sup_{a\neq a'} \frac{d_B(f(a),f(a'))}{d_A(a,a')}~.$
\end{definition}
\begin{definition}[Fréchet derivative \cite {func_an} Definition 2 section 7.2]
    Let $E$ and $F$ be two Banach spaces with norms $\lV\cdot\rV_E$ and $\lV\cdot\rV_F$ and $U$ an open subset of $E$. A function $f:U\rightarrow F$ is said to be Fréchet differentiable at $x_0\in U$ if and only if there exists a bounded linear operator $A:E\rightarrow F$ such that \[\lim_{\lV \z \rV_E\rightarrow 0} \frac{\lV f(x_0+\z) - f(x_0) - A\z\rV_F}{\lV \z \rV_E} = 0~.\]
    The differential is dentoed by $Df(x_0) = A$.
\end{definition}
\begin{theorem}[\cite {func_an} Section 7.3, Proposition 2]
\label{accroissement}
    Let $E$ and $F$ be two Banach spaces with norms $\lV\cdot\rV_E$ and $\lV\cdot\rV_F$. Consider an open subset $U\subset E$ and a function $f:U\rightarrow F$ that is Fréchet differentiable on $U$. Let $y\in U$ and suppose that there exists $\z\in E$ such that $y+\alpha \z\in U$ for all $0\leq \alpha \leq 1$. Then we have  
    \[
    \lV f(y)-f(y+\z)\rV_F  \leq \sup_{0\leq \alpha \leq 1} \lV Df(y + \alpha \z)\rVop \lV \z \rV_E\leq \sup_{x\in U} \lV Df(x)\rVop \lV \z \rV_E~.
    \]
\end{theorem}
\begin{proposition}
\label{prop:lip>diff}
    Consider $E$, $F$ two Banach spaces and an open and convex subset $U\subset E$. Let  $f:U\mapsto F$ be Lipschitz continuous at $U$ and Fréchet differentiable on $x_0\in E$. Then, \[\Lip(f) =\sup_{x,y\in U, x\neq y}\frac{\lV f(x) - f(y)\rV_F}{\lV x-y \rV_E} \geq \lV Df(x_0)\rVop~.\]
\end{proposition}
\begin{proof}
    Let $\z\in E$. By the convexity and openness of $U$, for $\alpha$ small enough, we have $x_0+\alpha z\in U$. Finally, we can write \begin{align*}
         \frac{\lV Df(x_0)[\z]\rV_F}{\lV \z \rV_E} &=
         \lim_{\alpha \rightarrow 0} \frac{\lV Df(x_0) [\alpha \z ]\rV_F}{\lV \alpha \z \rV_E}
         \\
         &\leq\lim_{\alpha\rightarrow 0} \frac{\lV f(x_0+\alpha \z) - f(x_0) - Df(x_0)[\alpha \z]\rV_F}{\lV \alpha \z \rV_E} + \lim_{\alpha\rightarrow 0}\frac{\lV f(x_0+\alpha \z) - f(x_0)\rV_F}{\lV \alpha \z \rV_E}\\
         &= 0+ \lim_{\alpha\rightarrow 0}\frac{\lV f(x_0+\alpha \z) - f(x_0)\rV_F}{\lV \alpha \z\rV_E}\\
         &\leq \Lip(f)~.
    \end{align*}
      Taking the supremum over all $\z\in E$ leads to the result.
\end{proof}
\begin{corollary}
\label{banach_mean}
    Consider $E$ and $F$ two Banach spaces. Let $U\subset E$ be open and convex and $f:U\mapsto F$ be Fréchet differentiable on $U$. Then $f$ is Lipschitz with $\Lip(f) =\sup_{x\in U} \lV Df(x)\rVop~.$
  
\end{corollary}
\begin{proof}
    Thanks to Theorem \ref{accroissement} and the convexity of $U$, it is clear that for all $x\neq y$ in $U$, we have $\frac{\lV f(x) - f(y)\rV_F}{\lV x-y \rV_E} \leq\sup_{x'\in U} \lV Df(x')\rVop$. Taking the supremum on $x$ and $y$ in $U$ leads to the inequality $\Lip(f)\leq \sup_{x'\in E} \lV Df(x')\rVop$. The second inequality is derived as follows. The function $f$ is Fréchet differentiable on every $x\in U$. Applying Proposition \ref{prop:lip>diff} leads to $\Lip(f)  \geq \lV Df(x)\rVop$. Taking the supremum over all $x\in U$ gives the second inequality.
\end{proof}

\section{Proofs}
\subsection{Lipschitz constant for Integral Kernels}
\label{proof:int_kernel}

\subsubsection{Proof of \cref{cor:integral_rkhs}}
\label{app:proofThm22}

\begin{proof} 
Let $v\in V^\perp$, then for all $x\in\X$, $S(v)(x) = \lb v,\theta(x)\rbL{P}{2}=0$, that is, $v\in \textit{Ker}(S)$. The reverse inclusion holds since $\Span(\theta(\X))^\perp = (\overline{\Span(\theta(\X))})^\perp $ (see \cite[Lemma 15.1.2]{hilbert}). We deduce that $V = \text{Ker}(S)^\perp$.\\
From Proposition \ref{integralrkhs}, we know that $S$ is an isometric isomorphism from $V = \text{Ker}(S)^\perp$ to $\HH$. Hence, $\HH = S(V) \subset S(\Lp{P}{2})$. 
    Since $V$ is closed, the orthogonal of $V$ is its complement \cite{lax2014functional}. 
Thus, for $u\in\Lp{P}{2}$, there exist unique $v_1\in V$ and $v_2\in V^\perp$ such that $u=v_1+v_2$. Observe that $S(u) = S(v_1)\in \HH$, and $S(V) = S(\Lp{P}{2})$. Furthermore \begin{align*}
         \lV S(\uu) \rV _{\HH} &=  \lV \uv_1 \rVL{P}{2} \\&\leq  \lV \uv_1 \rVL{P}{2} +  \lV \uv_2 \rVL{P}{2}\\
        &=  \lV  \uu  \rVL{P}{2}~.
    \end{align*}
    So, for all $\ha\in\HH$, $ \lV \ha \rVH = \inf\{ \lV \uu \rVL{P}{2}:\uu\in S^{-1}(\ha)\}.$
\end{proof}

\subsubsection{Proof of \cref{2 inf lip}}
\label{app:proofThm25}

\begin{proof}
    Recall that $\X$ is an open and convex subset of a Banach space $(E,\lV\cdot\rV_E)$. Let $x\in\X$. For $P$-almost $\om\in\Om$, there exists a bounded linear operator $D_x\phi( \om,x):E\rightarrow \RR$ such that \[\lim_{\lV \z \rVE\rightarrow 0} \frac{|\phi(\om,x+\z) - \phi( \om,x) - D_x\phi( \om,x)[\z]|}{\lV \z \rV_E} =\lim_{\lV \z \rVE\rightarrow 0} \frac{|r(\om)[z]|}{\lV \z \rV_E} = 0~,\]
    where $r(\om)[\z] = \phi(\om,x+\z) - \phi( \om,x) - D_x\phi( \om,x)[\z]$. For $\om\in\Om$ where $\phi( \om,x)$ is not differentiable we have chosen $ D_x\phi( \om,x)=0$. As a consequence, $r(\om)[z]$ is well defined at each $\om \in \Om$, and at $z$ in the open neighborhood of $0$ in $E$ composed of vectors $z\in E$ such that $x+z\in\X$ (it is open as the pre-image of an open set by a continuous function). 
    
    We will show that the feature map $\vphi : x\mapsto k_x$ is Fréchet differentiable at any $x \in \X$. We define the operator $S:\Lp{P}{2}\rightarrow\HH$ as in Theorem \ref{integralrkhs}, that is \[S(a)(x) = \lb a,\phi(\cdot,x) \rbL{P}{2}~.\]
 Note that $S(\phi(\cdot,x))(x') = \lb \phi(\cdot,x),\phi(\cdot,x')\rbL{P}{2} = k(x,x') = \vphi(x)(x')$, so $S(\phi(\cdot,x)) = \vphi(x)$.
    Thus, we can write
    \begin{align*}
        \vphi(x+\z) &= S(\phi(\cdot,x+z))\\
        &= S\left(\om\mapsto (r(\om)[\z]+\phi( \om,x) + D_x\phi( \om,x)[\z])\right)\\
        &= \underbrace{S(r(\cdot)[\z])}_{R[\z]}+\underbrace{S(\phi(\cdot,x))}_{\vphi(x)} + \underbrace{S( D_x\phi( \cdot,x)[\z])}_{A[\z]} ~.\\
    \end{align*}
    The precedent decomposition is valid since $ D_x\phi( \cdot,x)[\z]\in \Lp{P}{2}$ because $D_x\phi( \cdot,x)[\z]$ is measurable by the hypothesis (iii) and $|D_x\phi( \om,x)[\z]|\leq F(\om)\lV z \rVE$ 
    $P$-a.e. with $F\in\Lp{P}{2}$ by the hypothesis (i) and Proposition \ref{prop:lip>diff}. 
    
    Then, $r(\cdot)[\z] = \phi(\cdot,x+z) - \phi(\cdot,x) -  D_x\phi( \cdot,x)[\z] \in \Lp{P}{2}$. We could have used the notations $r_x$, $R_x$ and $A_x$ to precise that those maps depend on $x$, but to simplify the notation we omit it. Remark also that $R[z] = \vphi(x+z) - \vphi(x) - A[z]$, so we want to show that $A$ is the Fréchet derivation of $\vphi$ at $x$. We need to show that \begin{enumerate}[(1)]
        \item $A$ is a bounded linear operator,
        \item $\lim_{\lV \z \rVE\rightarrow 0}\frac{\lV R[\z]\rVH}{\lV \z \rVE}=0$.
    \end{enumerate}
    First we prove (1). Let $\z_1,\z_2\in E$ and $\alpha\in\RR$. By the linearity of $D_x\phi( \om,x)$ and the linearity of the integral, we have
    \begin{align*}
    A[\z_1+\alpha \z_2] &=S(D_x\phi( \cdot,x)[\z_1 + \alpha \z_2])\\
    &=S (D_x\phi(\cdot,x)[\z_1])+\alpha S(D_x\phi(\cdot,x) [\z_2])\\
    &= A[\z_1]+\alpha A[\z_2]~.
    \end{align*}
    Furthermore, for all $\z\in\BE$, we have
    \begin{align}
        \lV A[\z]\rVH&=\lV S(D_x\phi( \cdot,x)[\z])\rVH \nonumber \\
        &\leq \lV D_x\phi( \cdot,x)[\z]\rVL{P}{2}\tag{Using the norm defined in \ref{integralrkhs}} \nonumber \\ 
        &\leq\lV F\rVL{P}{2}. 
        \label{eq:AnormBYFL2}
    \end{align}
Now we prove (2). We will apply the theorem of dominated convergence. Let $(\z_n)$ be a sequence that converges to $0$ in $E$. Without loss of generality, we can choose $(z_n)$ in the open neighborhood of $0$ where $r(\om)[\cdot]$ is defined. Fix $x \in \X$, and for $\om\in\Om$ define $\beta_n(\om) = \left(\frac{r(\om)[\z_n]}{\lV \z_n \rVE}\right)^2$. 

Each $\beta_n$ is measurable as a polynomial of measurable functions. The sequence $(\beta_n)$ converges almost surely to $0$ since $\phi(\om,\cdot)$ is $P$-almost surely differentiable at $x$. Finally, we need to show the domination hypothesis:
\begin{align*}
    \sqrt{\beta_n(\om)} &= \frac{|r(\om)[\z_n]|}{\lV \z_n \rVE}\\
    &\leq  \frac{1}{\lV \z_n \rVE}(|r(\om)[\z_n]+\phi( \om,x) - \phi(\om,x+\z_n)|+|\phi( \om,x) - \phi(\om,x+\z_n)|)\\
    &\leq  \frac{1}{\lV \z_n \rVE}(|D_x\phi( \om,x)[\z_n]|+\Lip(\phi(\om,\cdot))\lV \z_n \rVE)\\
    &\leq\frac{1}{\lV \z \rVE}(\lV D_x\phi( \om,x)\rVop\lV \z_n \rVE+\Lip(\phi(\om,\cdot)) \lV \z_n \rVE)\\
    & \leq  2\Lip(\phi(\om,\cdot))\tag{By Proposition \ref{prop:lip>diff}}~.
\end{align*}
Thus, for $P$-almost all $\om\in\Om$, \[\sqrt{\beta_n(\om)}\leq 2F(\om)~.\]

So, for all $n\in\mathbb{N}$, $|\beta_n|$ is dominated by an integrable function $4F^2$.

Finally we have that \begin{align*}
    \frac{\lV R[\z_n]\rVH}{\lV \z_n \rVE}&=\frac{1}{\lV \z_n \rVE}\lV S(r(\cdot)[\z_n])\rVH \\
    &\leq \frac{1}{\lV \z_n \rVE}\lV  r(\cdot)[\z_n]\rVL{P}{2}\tag{By \ref{integralrkhs}}\\
    &=\left(\int_\Om \frac{r(\om)[\z_n]^2}{\lV \z_n \rVE^2}dP(\om)\right)^{\frac{1}{2}}\\
    &=\bigg(\underbrace{\int_\Om \beta_n(\om)dP(\om)}_{\xrightarrow[n\rightarrow\infty]{} 0}\bigg)^{\frac{1}{2}}~.
\end{align*}

It is true for all sequences $(\z_n)$  converging to $0$, so we showed (2) :\[\lim_{\lV \z \rVE\rightarrow 0}\frac{\lV R[\z_n]\rVH}{\lV \z \rVE}=0~.\]
We deduce that $\vphi$ is Fréchet differentiable on $\X$ and for all $\z\in E$, we have 
\begin{equation}\label{eq:DvphiEqDSphiz}
D_x\vphi(x)[\z] = S(D_x\phi( \cdot,x)[\z])~.  
\end{equation}

As we saw earlier in \eqref{eq:AnormBYFL2}, 
for every $x\in\X$ and $\z\in\BE$, we have 
$ \lV D_x\vphi(x)[\z]\rVH\leq \lV F\rVL{P}{2}$ and so $ x \mapsto \lV D_x\vphi(x)[\z]\rVop$ is bounded. As $\X$ is open and convex, according to Corollary \ref{banach_mean}, we have that $\vphi$ is Lipschitz with $\Lip(\vphi) = \sup_{x\in\X} \lV D_x\vphi(x)\rVop$. This value is not tractable but fortunately, we can express it as a norm in $\Lp{P}{2}$. By Theorem \ref{integralrkhs} and \eqref{eq:DvphiEqDSphiz}, we know that \[\lV D_x\vphi(x)[\z]\rVH\leq \lV D_x\phi(\cdot,x)[\z]\rVL{P}{2}~.\] This is also an equality. 
To see that, we need to show that $D_x\phi(\cdot,x)[\z] \in (\textit{Ker}(S))^\perp$ 
as we know that $S:(\textit{Ker}(S))^\perp \rightarrow \HH$ is an isometry from Theorem \ref{integral_kernel}. Let $\uu\in \textit{Ker}(S)$. For every $x\in\X$, $S(\uu)(x) =0$, and hence $S(\uu)$ is Fréchet differentiable with  $D_xS(\uu)(x) =0$. 
The key element is showing that $B[\z] = \lb D_x\phi(\cdot,x)[\z], \uu\rbL{P}{2}$ is the derivative of $S(
\uu)$ at $x$. First, we see that $B$ is bounded and linear since $|B[\z]|\leq \lV F \rVL{P}{2}  \lV \uu \rVL{P}{2}$ by the Cauchy-Schwarz inequality and the linearity of $A$ in \eqref{eq:AnormBYFL2}.
Secondly, \begin{align*}
0\leq \frac{|S(\uu)(x+\z) - S(\uu)(x) - B[\z]|}{\lV \z \rV_E}&=\frac{1}{\lV \z \rV_E}\left|\lb \phi(\cdot,x+z) - \phi(\cdot,x) - D_x\phi(\cdot,x)[\z], \uu \rb_\Lp{P}{2}\right|\\
&\leq  \frac{\lV r(\cdot)[\z]\rV_\Lp{P}{2}}{\lV z \rV_E}\lV u\rV_\Lp{P}{2} \xrightarrow[\lV z \rV_E \rightarrow 0]{} 0 ~.
\end{align*}
We deduce that $D_xS(\uu)(x)[\z] =  \lb D_x\phi(\cdot,x)[\z], \uu\rb_\Lp{P}{2} =0$ for every $x\in\X$ and $z\in E$. Hence, $D_x\phi(\cdot,x)[\z]$ is 
orthogonal to $\textit{Ker}(S)$ and we can use the isometry property 
 \begin{align}
     \lV D\vphi(x)\rVop &= \sup_{\z\in \BE} \lV D\vphi(x)[\z]\rVH \nonumber \\
     &=\sup_{\z\in \BE} \lV D_x\phi(\cdot,x)[\z]\rVL{P}{2}~. \label{eq:isometryDphi}
 \end{align}

As a consequence of \eqref{eq:isometryDphi} and Corollary \ref{banach_mean}, we have
 \[\Lip(\vphi) = \sup_{x,\z\in \BE} \lV D_x\phi(\cdot,x)[\z]\rVL{P}{2}.\]
 The second formulation can be derived from the facts that $\vphi$ is Fréchet differentiable. The precedent implies that \begin{align*}
     \frac{f(x+tz) - f(x)}{t} = \lb f, \frac{\vphi(x)-\vphi(x+tz)}{t}\rbH
     \xrightarrow[t\rightarrow 0^+]{} \lb f, D_x\vphi(x)[z]\rbH  ~,
 \end{align*}
 by continuity of the inner product. So, we have $D_xf(x)[z] = \lb f, D_x\vphi(x)[z]\rbH$. Remark that, by symmetry, we have $D_x\vphi(x')[z] = D_xk(x',\cdot)[z] = D_yk(\cdot,x')[z]$, for all $x'\in\X$ and $z\in\BE$. Finally, \begin{align*}
     \lV D_x\vphi(x')[z] \rVH^2 &= \lb D_x\vphi(x')[z], D_x\vphi(x')[z]\rbH\\
     &= \lb D_yk(\cdot,x')[z],D_x\vphi(x')[z] \rbH\\
     &= D_xD_yk(x',x')[z,z] ~.
 \end{align*}
\end{proof}

\subsubsection{Proof of \cref{2 inf}}
\label{app:proofThm26}
\begin{proof}
    Suppose $\sup_{x\in\X,\z\in \BE}\lV D_x\phi(\om,x)[\z] \rVL{P}{2}=+\infty$ and let $M>0$. There exist $x\in\X$ and $z\in\BE$ such that $\lV D_x\phi(\om,x)[\z] \rVL{P}{2}>M$. Since $\X$ is open, there exists $t_0>0$ such that $x+t_0 z\in\X$. Furthermore, define $t_n=\frac{t_0}{n+1}$ for all $n\geq 1$. The sequence $(t_n)$ is such that $t_n\xrightarrow{} 0$ and $y_n = x+t_n z \in\X$ for all $n\in \mathbb{N}$. We have \begin{align*}
        \lV \vphi(x) - \vphi(y_n)\rVH^2 &= k(x,x) + k(y_n,y_n) -2k(x,y_n)\\
        &= \int_\Om \left(\phi(\om,x)^2 + \phi(\om,y_n)^2 - 2\phi(\om,x)\phi(\om,y_n) \right) dP(\om)\\
        &= \int_\Om \left(\phi(\om,x) - \phi(\om,y_n) \right)^2 dP(\om) ~.
    \end{align*}
    Thus, by Fatou's Lemma, \begin{align*}
        \lim \inf\frac{\lV\vphi(x) - \vphi(y_n)\rVH^2}{t_n^2} &= \lim \inf \int_\Om \frac{\left(\phi(\om,x) - \phi(\om,y_n) \right)^2}{t_n^2}dP(\om)\\
        &\geq \int_\Om \lim \inf \frac{\left(\phi(\om,x) - \phi(\om,y_n) \right)^2}{t_n^2}dP(\om)\\
        &= \int_\Om \left(D\phi(\om,x)[z]\right)^2 dP(\om)\\
        &>M^2 ~.
    \end{align*}
    We deduce that there exists $n\in\mathbb{N}$ such that $\frac{\lV\vphi(x) - \vphi(y_n)\rVH}{\lV x - y_n\rV}>M$. Finally, \[\Lip(\vphi) = \sup_{x,y\in\X,x\neq y}\frac{\lV\vphi(x) - \vphi(y)\rVH}{\lV x - y \rV_E} = +\infty~.\]
\end{proof}

\subsection{Lipschitz constant for Neural Network kernels}
\subsubsection{Proof of \cref{cor:lip_NN}}
\label{app:lip_NN}

\begin{proof}
We want to show that $\phi\in\LpX{P}{2}$. By hypothesis, $\sigma$ is Lipschitz-continuous,  $\EP \lV w \rVe^2<\infty$ and $\EP b ^2<\infty$. We want to show that for every $x\in\X$, we have $\phi(\cdot,x) = (w,b)\mapsto \sigma(w^\T x+b)\in\Lp{P}{2}$. Fix $x\in\X$. The function $\phi(\cdot,x)$ is measurable because it is continuous. Then we need to show that it is square-integrable. By the Lipschitz-continuity of $\sigma$, we know that for every $(w,b)\in\RR^d\times\RR$, $|\phi((w,b),x)|\leq |\phi((0,0),x)| + \Lip(\sigma)(\lV x \rVe\lV w \rVe +|b|)$. We deduce $\EP \phi((w,b),x)^2<\infty$ and $\phi\in\LpX{P}{2}$.

Second, we show that $\Lip(h)\leq \Lip(\sigma)\sqrt{\EP \lV w \rVe^2 \EP u^2}$ for every $u\in\Lp{P}{2}$.
Recall that $h(x) = \int_\Om u(w,b) \sigma(w^\T  x+b)dP(w,b)$. Note that for every $(w,b)\in\RR^d\times\RR$ and  $x,x'\in\RR^d$, \begin{align*}
        |\phi((w,b),x) - \phi((w,b),x') |&= |\sigma(w^\T x +b)  - \sigma(w^\T x' +b)  |\\
        &\leq \Lip(\sigma) |w^\T x +b - (w^\T x' +b)|\\
        &\leq  \Lip(\sigma)\lV w \rVe \lV x-x'\rVe~.
    \end{align*}
    By proposition \ref{prop:upperbound}, we have that the feature map of the integral kernel $\vphi$ is Lipschitz-continuous with $\Lip(\vphi)\leq \Lip(\sigma)\sqrt{ \EP \lV w \rVe^2}$. Using the operator $S$ defined in \Cref{cor:integral_rkhs}, we have that $h = S(u)$. By the characterization of the norm in the RKHS in \eqref{eq:lipRKHSchar}, we get \begin{align*}
        \Lip(h)&\leq \Lip(\vphi)\lV h \rVH\\
        &\leq \Lip(\vphi)\lV u \rVL{P}{2}\\
& \leq \Lip(\sigma)\sqrt{\EP \lV w \rVe^2 \EP u^2}~.
        \end{align*}
\end{proof}

\subsubsection{Proof of \cref{prop:RNNlipinf}}
\label{app:lip_NN2}

Suppose, as in the corollary, that $\sigma'$ is measurable and $P$ has a continuous density. We will show that the assumptions (i), (ii) and (iii) in Theorem \ref{2 inf lip} are true.
        
        \textbf{Assumption (i):}         
        It is shown in \Cref{app:lip_NN} that the $  F(\om )$ in the first assumption of \Cref{2 inf lip}
        corresponds to $ F(\om ) =     \Lip(\sigma)\lV w \rVe$ due to $\Lip(\phi((w,b),\cdot)\leq \Lip(\sigma)\lV w \rVe$ for all $(w,b)\in\RR^d\times \RR$. It is clear that  $F  \in \Lp{P}{2}$. Hence, this holds with both cases.

    \textbf{Assumption (ii):}

    \textbf{Case $\sigma$ is differentiable on $\RR$:} We have that $\phi(\om,\cdot)$ is differentiable everywhere for all $\om\in\Om$.
    
    \textbf{Case where $\sigma'$ is measurable and $P$ is absolutely continuous:}
    Fix $x\in\X$. We denote by $\Gamma$ the set of points in $\RR$ where $\sigma$ is not differentiable and by $\Gamma_x$ the set of points $\om$ in $\Om$ where $\phi(\om,\cdot)$ is not differentiable at $x$. We vectorize an element $(w,b)\in\Om$ by $\bar{\om} = (w_1,...,w_d,b)\in\RR^{d+1}$, where $w=(w_1,...w_d)$. Furthermore, we write $\bar{x} = (x_1,...,x_d,1)$ where $x = (x_1,...,x_d)$. With those new notations, we can rewrite $\phi(\om,x) = \sigma(\bar{\om}^\T \bar{x})$. It leads to write $\Gamma_x$ as the set $\{ \bar{\om} : \bar{\om}^\T\bar{x}\in\Gamma\}$. We will write it in
    in another orthonormal basis aiming to compute its Lebesgue measure. Let the first component be $e_1 = \frac{\bar{x}}{\lV\bar{x}\rVe}$. We can complete this to form an orthonormal basis $(e_1,...,e_{d+1})$. In this basis, we have $x=(\lV x \rV,0,...,0)$ and we denote the coordinates of $\bar{\om}$ in this basis by $(a_1,...,a_{d+1})$. The condition $\lb \bar{\om}, \bar{x}\rb \in \Gamma$ is equivalent to $a_1\lV \bar{x} \rV \in\Gamma$ since the inner product is independent of the orthonormal basis. Thus, $a_1\in \frac{\Gamma}{\lV\bar{x}\rVe}$ since $\lV\bar{x}\rVe\neq0$ by definition.

    Finally, in this basis, we can write $\Gamma_x=\frac{\Gamma}{\lV\bar{x}\rVe} \times \mathbb{R}^d$. 
    From the Lipschitz-continuity of $\sigma$,  $\Gamma$ is contained in a Lebesgue-measurable set $\Gamma_0$ with $\lambda^1\left(\frac{\Gamma_0}{\lV\bar{x}\rVe}\right) = 0$. Therefore $\Gamma_x$
    is contained in the set $\frac{\Gamma_0}{\lV\bar{x}\rVe} \times \mathbb{R}^d$ whose Lebesgue measure is zero \cite[Theorem 8.6]{rudin1987}. Besides, from the assumption that $P$ is absolutely continuous w.r.t. the Lebesgue measure $\lambda^{d+1}$, we have $P\left(\frac{\Gamma_0}{\lV\bar{x}\rVe} \times \mathbb{R}^d\right) = 0$.

    \textbf{Assumption (iii):} The application 
    $D_x \phi(\cdot,x)[z] = (w,b)\mapsto  ( w^\T z) \sigma'(w^\T x +b)$ 
    is measurable, since $\sigma'$ is measurable by hypothesis in both cases.

\label{app:proofProp26}

\subsubsection{Proof of \cref{propRNN}}
\label{app:lip_NN3}

\begin{proof}
We first need to apply Corollary \ref{cor:lip_NN} aiming to show that $\phi\in\LpX{P}{2}$ and that the assumptions of Theorem \ref{2 inf lip} are valid. $\sigma$ is Lipschitz-continuous by hypothesis. 
Then, $(w,b)\mapsto \lV w \rVe + | b |$ is square integrable since 
$\EP \lV w \rVe^2 = d \gamma^2$ (Gaussian law) and $\EP b^2<\infty$ by hypothesis. Then $\phi\in\LpX{P}{2}$. Furthermore, by hypothesis, $\sigma'$ is measurable.
We will show that the measure $P$ is absolutely continuous w.r.t. the Lebesgues measure $\lambda^{d+1}$ on $\RR^{d+1}$. Let $E\in \mathcal{B}(\RR^{d}\times \RR)$. We note $E_x=\{y\in\RR:(x,y)\in E\}$ and $E^y=\{x\in\RR^d:(x,y)\in E\}$. We suppose that $\lambda^{d+1}(E)=(\lambda^d\otimes\lambda^1)(E) = 0$ and we aim to prove that $P(E)=0$. According to \cite[Theorem 34.A p.147]{measure_theory}, we have $\lambda^1$-almost surely $\lambda^d(E_x)=0$ or $\lambda^d$-almost surely $\lambda^1(E)=0$. Hence, by the absolute continuity of $p_b$ and $\mu_\gamma^d$, we have $\mu_\gamma^d$-almost surely $p_b(E_x)=0$ or $p_b$-almost surely $\mu_\gamma^d(E^y)=0$. By the same theorem, $P(E) = (\mu_\gamma^d\otimes p_b)(E) = 0$. 
According to Corollary \ref{cor:lip_NN}, the assumptions of Theorem \ref{2 inf lip} hold. 
Next, we compute $\EP [ (D_x\sigma (w^\T x+b ) [\z])^2]$ to get an expression of the Lipschitz constant of $\vphi$.

\textbf{Case $x\neq 0$:} Define the functions $\beta(a) = \EPP [(\zeta^2-\gamma^2)\sigma'(a\zeta+b)^2]$ and $\alpha(a) = \EPP[\sigma'(a\zeta+b)^2]$. We will show that for all $\z = (\z_1,...,\z_d)$ and $x = (x_1,...,x_d)$ in $\RR^d$, \[\int_\Om ( D_x\phi(\om,x)[\z] )^2dP(\om) = \frac{(x^\T \z)^2}{\lV x \rVe^2} \beta(\lV x\rVe) + \lV z \rVe^2 \alpha(\lV x \rVe) ~.\]
First observe that
\begin{align*}
    \EP \left[(D_x\phi(\om,x)[\z])^2 \right]&= \EP \left[\left(w^\T \z \sigma'(w^\T  x+b)\right)^2\right]\\
    &= \EP \left[ \sum_{1\leq k,l\leq d} \z_k \z_l\underbrace{w_{k} w_{l}  \sigma'(w^\T  x+b)^2}_{G_{k,l}}\right]\\
    &=  \sum_{1\leq k,l\leq d} \z_k \z_l\EP \left[\underbrace{w_{k} w_{l}  \sigma'(w^\T  x+b)^2}_{G_{k,l}}\right]~.
    \end{align*}
     Attempting to compute $\EP[G_{k,l}]$ we will proceed with an orthogonal decomposition of $w$ along the direction $x / \| x \|$ and the orthogonally complementary subspace: $w = \frac{x }{ \| x \|} \zeta +  \eta$ where $\lb \eta , x \rb = 0$. Since \[\zeta = w^\T \frac{x}{\lV x \rVe} = \sum_{k=1}^d w_k \frac{x_k}{\lV x \rVe}~,\] it follows that $\zeta$ follows a Gaussian distribution with $\EP \zeta = 0$ and $\EP \zeta^2 = \sum_{k=1}^d w_k^2 \frac{x_k^2}{\lV x \rVe^2} = \gamma^2$, i.e. $\zeta\sim\mathcal{N}(0,\gamma^2)$.
   Let $\eta_k$ be the $k$-th component of $\eta$, we have $\eta_k = w_k - \frac{x_k}{\lV x \rVe} \zeta$. Hence, $\eta_k$ is Gaussian with $\EP \eta_k = 0$ and \begin{align*}
         \EP \eta_k^2 &= \EP w_k^2 + \frac{x_k^2}{\lV x \rVe^2}\EP \zeta^2 - 2  \frac{x_k}{\lV x \rVe} \EP \zeta w_k\\
         &= \gamma^2 + \frac{x_k^2}{\lV x \rVe^2}\gamma^2 - 2  \frac{x_k^2}{\lV x \rVe^2} \EP \sum_{l=1}^d w_l w_k \\
         &= \gamma^2 \left( 1 -  \frac{x_k^2}{\lV x \rVe^2}\right)~.
     \end{align*}
     Let us re-parameterize $\eta=(\eta_1,\cdots,\eta_d)$ with the random variables $(\zeta_1,\cdots,\zeta_d)$, i.e.
     \[
        \eta_k  = \sqrt{1-\frac{x_k^2}{\lV x \rVe^2}} \zeta_k,  \quad \zeta_k \sim \mathcal{N}(0,\gamma^2), \quad \forall 1\leq k \leq d .
     \]
Finally, observe that the Gaussian random variables $\zeta_k$ and $\zeta$ are decorrelated and hence they are independent,\begin{align*}
    \EP \zeta \eta_k &= \zeta \left(w_k - \frac{x_k}{\lV x \rVe}\right)\\
    &= \EP w_k \zeta - \frac{x_k}{\lV x \rVe}\EP \zeta^2\\
    &= \frac{x_k }{\lV x \rVe} \gamma^2 - \frac{x_k }{\lV x \rVe} \gamma^2\\
    &=0~.
\end{align*}
Now, we compute 
    \begin{align*}
        \EP[G_{k,l}] &= \EP\left[w_k w_l \sigma'(w^\T  x+b)^2\right]\\
        &=\EP\left[\left(\frac{x_k}{\lV x \rVe}\zeta +\sqrt{1-\frac{x_k^2}{\lV x \rVe^2}}\zeta_k\right)\left(\frac{x_l}{\lV x \rVe}\zeta +\sqrt{1-\frac{x_l^2}{\lV x \rVe^2}}\zeta_l\right)\sigma'(\zeta \lV x \rV + b)^2\right]\\
        &= \EP\left[\left(\frac{x_kx_l}{\lV x \rVe^2}\zeta^2 + \underbrace{\sqrt{1-\frac{x_k^2}{\lV x \rVe^2}}\sqrt{1-\frac{x_l^2}{\lV x \rVe^2}}\zeta_k\zeta_l}_{A}\right)\sigma'(\zeta \lV x \rV + b )^2\right]\\
        &=\EP\left[\frac{x_kx_l}{\lV x \rVe^2}\zeta^2 \sigma'(\zeta \lV x\rVe+b)^2\right] + \EP[A]\EP[\sigma'(\zeta \lV x \rV + b )^2] ~.
    \end{align*}
    The last equality is due to the decorrelation (Gram-Schmidt) and consequently the independence (Gaussian) of $\zeta_k,\zeta_l$ with $\zeta$. 
    \begin{align*}
        \EP[A] &= \EP\left[\left(w_k - \frac{x_k}{\lV x \rVe}\zeta\right)\left(w_l - \frac{x_l}{\lV x \rVe}\zeta\right)\right]\\&= \underbrace{\EP[w_kw_l]}_{\gamma^2\delta_{k,l}} + \frac{x_kx_l}{\lV x \rVe^2}\underbrace{\EP[\zeta^2]}_{\gamma^2} - \frac{x_k}{\lV x \rVe}\underbrace{\EP[w_l\zeta]}_{\gamma^2\frac{x_l}{\lV x \rVe}} - \frac{x_l}{\lV x \rVe}\underbrace{\EP[w_k\zeta]}_{\gamma^2\frac{x_k}{\lV x \rVe}}\\
        &=\gamma^2\left(\delta_{k,l} - \frac{x_kx_l}{\lV x \rVe^2}\right)~.
    \end{align*}
    Now we can compute $ \EP[G_{k,l}]$:
    \begin{align*}
        \EP[G_{k,l}] &= \frac{x_kx_l}{\lV x \rVe^2}\EP[(\zeta^2-\gamma^2)\sigma'(\zeta\lV x\rVe + b )^2] + \gamma^2\delta_{k,l} \EP[\sigma'(\zeta\lV x\rVe + b)^2]\\ &= \frac{x_kx_l}{\lV x \rVe^2}\beta(\lV x \rVe) + \gamma^2\delta_{k,l} \alpha(\lV x \rVe)~.
    \end{align*}
    Finally we obtain
    \begin{align*}
        \lV D_x\phi(\cdot,x)[\z]\rVL{P}{2}^2&=\sum_{1\leq k,l\leq d}\z_k\z_l \EP[ G_{k,l} ] \\
        &= \sum_{1\leq k,l\leq d} \frac{\z_k\z_lx_kx_l}{\lV x \rVe^2}\beta(\lV x \rVe) + z_k z_l\gamma^2\delta_{k,l}\alpha(\lV x \rVe)\\ &= \sum_{k=1}^d \z_kx_k \frac{\z^\T  x}{\lV x \rVe^2}\beta(\lV x \rVe) + \sum_{k=1}^d  \z_k^2\gamma^2 \alpha(\lV x \rVe)\\
        &= \frac{(\z^\T  x)^2}{\lV x \rVe^2}\beta(\lV x \rVe) + \lV \z \rVe^2 \gamma^2
        \alpha(\lV x \rVe)~.
    \end{align*}
      For all $x\in \X,\z\in E$ with $x\neq 0$, we have \begin{align*}
     \sup_{\z\in\BR}\lV D_x\phi(\cdot,x)[\z]\rVL{P}{2}^2&=\sup_{\z\in\BR} \frac{(x^\T  \z)^2}{\lV x \rVe^2} \beta(\lV x\rVe) + \lV z \rVe^2 \gamma^2\alpha(\lV x \rVe)\\
    &= \frac{\left(x^\T \frac{x}{\lV x\rVe}\right)^2}{\lV x \rVe^2} \beta(\lV x\rVe) + \lV \frac{x}{\lV x\rVe} \rVe^2 \gamma^2\alpha (\lV x \rVe) \\
    &= \gamma^2\alpha(\lV x \rVe) + \beta(\lV x \rVe)\\
    &=  \underbrace{\EP \zeta^2[\sigma'(\lV x \rVe\zeta+b)^2]}_{\nu(x)}~.
    \end{align*}

    \textbf{Case $x=0$:}
    Then $\EP G_{k,l} = \EP[w_k w_l] \EP [\sigma'(b)^2] = \gamma^2 \delta_{k,l} \EP [\sigma'(b)^2]$. We deduce that \begin{align*}
    \sup_{\z\in\BR}\lV D_x \phi(\cdot,x)[z]\rVL{P}{2}^2 &= \sup_{\z\in\BR}\sum_{k=1}^d \sum_{l=1}^d \EP z_k z_l G_{k,l}\\
    &= \sup_{\z\in\BR}\sum_{k=1}^d z_k^2  \EP \zeta^2[\sigma'(b)^2] \\
    &=  \sup_{\z\in\BR}\lV z \rVe^2 \EP \zeta^2[\sigma'(b)^2]\\
    &=  \underbrace{\EP [ \zeta^2\sigma'(\lV 0 \rVe\zeta+b)^2]}_{\nu(0)}~.\\
    \end{align*}
    Thus we can write  
\begin{align*}
     \sup_{x\in\X,\z\in\BR}\lV D_x \phi(\cdot,x)[z]\rVL{P}{2}^2 & = \sup_{x\in\X}\sup_{\z\in\BR}\lV D_x \phi(\cdot,x)[z]\rVL{P}{2}^2\\
     &= \sup_{x\in\X} \nu(x) ~.
\end{align*}
    Finally we have,
    \begin{align*}
         \sup_{x\in\X,\z\in\BE} \lV D_x\phi(\cdot,x)[\z]\rVL{P}{2} =\sup_{x\in\X}\sqrt{ \EPP [\zeta^2\sigma'(\lV x \rVe\zeta+b)^2]} ~.
         \end{align*}

\end{proof}
\subsubsection{Proof of \cref{THMRFF}}
\label{thmRFF}
\begin{proof}
\noindent\textbf{Case $\E_{p_w}\|w\|^2<+\infty$.} First, we show that \cref{2 inf lip} can be applied using \cref{prop:RNNlipinf}. 

1. The application $(w,b)\mapsto\lV  w\rVe + |b|\in\Lp{P}{2}$ since $\E_{p_b} |b|^2<\infty$ and $\E_{p_w} \lV  w\rVe^2<\infty$.

2. The function $\sigma = \sqrt{2\kappa(0)}\cos$ is differentiable in $\RR$ and $\sigma' = -\sqrt{2\kappa(0)}\sin$ is measurable as it is continuous.

Applying \cref{2 inf lip}, we get \begin{align*}
    \Lip(\vphi) &= \sup_{x\in\X, z\in \BR}\sqrt{\E_{w\sim p_w, b\sim p_b} 2\kappa(0)(w^\T z)^2 \sin(w^\T x + b)^2}\\
    &= \sup_{z\in\BR} \sqrt{\E_{w\sim p_w} \kappa(0)(w^\T z)^2}\\
    &= \sup_{z\in\BR} \sqrt{\kappa(0)z^\T \E_{w\sim p_w} [w w^\T] z }\\
    &= \sqrt{\kappa(0)\lambda_{\max}(\text{Cov}(w))} ~.
\end{align*}
The precedent is valid since $\E_{p_w} w  =0$. Indeed $\E_{p_w} \lV w \rV  < \infty$ and $\kappa$ is symmetric.

\noindent Recall that, for $t\in\RR^d$,
\[
\kappa(t)=\int_{\RR^d} e^{jw^\top t}\,d\mu(w)=\kappa(0)\,\E_{p_w}\bigl[\cos(w^\top t)\bigr]~.
\]
Assume $\E_{p_w}\|w\|^2<\infty$. Since $|\cos(\cdot)|\le 1$ and $|\sin(\cdot)|\le 1$, for each $i,j\in\{1,\dots,d\}$ we have
\[
\Bigl|w_i \sin(w^\top t)\Bigr|\le |w_i|~,\qquad
\Bigl|w_i w_j \cos(w^\top t)\Bigr|\le |w_i w_j|~,
\]
and $|w_i|,|w_i w_j|\in L^1(p_w)$ by Cauchy--Schwarz and $\E_{p_w}\|w\|^2<\infty$.
Therefore, by dominated convergence, we may differentiate under the expectation:
\begin{align*}
\nabla \kappa(t)
&= \kappa(0) \nabla \E_{p_w}\bigl[\cos(w^\top t)\bigr]
= -\kappa(0)\,\E_{p_w}\bigl[w\,\sin(w^\top t)\bigr],\\
\nabla^2 \kappa(t)
&= -\kappa(0)\E_{p_w}\bigl[ w w^\top \cos(w^\top t)\bigr].
\end{align*}
Evaluating at $t=0$ and using $\cos(0)=1$ yields
\[
\nabla^2 \kappa(0)= -\kappa(0) \E_{p_w}[w w^\top]=-\kappa(0)\,\Cov(w),
\]
because $\E[w]=0$. Consequently,
\[
\kappa(0)\,\lambda_{\max}(\Cov(w))=\lambda_{\max}\!\bigl(-\nabla^2\kappa(0)\bigr),
\]
so the Lipschitz constant obtained above can be equivalently written as
\[
\Lip(\vphi)=\sqrt{\kappa(0)\,\lambda_{\max}(\Cov(w))}
= \sqrt{\lambda_{\max}\!\bigl(-\nabla^2 \kappa(0)\bigr)}.
\]

\noindent\textbf{Case $\E_{p_w}\|w\|^2=+\infty$.}
Let $\omega=(w,b)$ and consider the random feature integrand
\[
\phi(\omega,x)=\sqrt{2\kappa(0)}\cos(w^\top x+b).
\]
For every $x\in\X$, $\phi(\omega,\cdot)$ is Fr\'echet differentiable at $x$ with
\[
D_x\phi(\omega,x)[z]=-\sqrt{2\kappa(0)}\,(w^\top z)\sin(w^\top x+b),
\qquad z\in\RR^d,
\]
and $(\omega,x)\mapsto D_x\phi(\omega,x)[z]$ is measurable for each fixed $z$ (continuity in $(w,b,x)$).

Now assume $\E_{p_w}\|w\|^2=+\infty$. Since $\|w\|^2=\sum_{i=1}^d w_i^2$, there exists an index $i_0$ such that
$\E_{p_w}[w_{i_0}^2]=+\infty$.
Let $z=e_{i_0}\in\RR^d$ (unit vector). Using $\sin^2(\cdot)\le 1$ and Fubini/Tonelli,
\begin{align*}
\left\|D_x\phi(\cdot,x)[e_{i_0}]\right\|_{L^2(P)}^2
&= \E_{w\sim p_w,b\sim p_b}\Bigl[2\kappa(0)\,\sin^2(w^\top x+b)\,w_{i_0}^2\Bigr]\\
&= \E_{p_w}\Bigl[2\kappa(0)\,w_{i_0}^2\,\E_{b\sim p_b}[\sin^2(w^\top x+b)]\Bigr]. 
\end{align*}

Since $b\sim \mathcal{U}[0,2\pi]$, we have $\E_{p_b}[\sin^2(\theta+b)]=\frac12$ for every $\theta\in\RR$, hence
\[
\left\|D_x\phi(\cdot,x)[e_{i_0}]\right\|_{L^2(P)}^2
= \kappa(0)\,\E_{p_w}[w_{i_0}^2]=+\infty,
\]
for every $x\in\X$. Therefore,
\[
\sup_{x\in\X,\;z\in\BE}\left\|D_x\phi(\cdot,x)[z]\right\|_{L^2(P)}=+\infty.
\]
By \cref{2 inf}, this implies $\Lip(\vphi)=+\infty$, which concludes the proof.
\end{proof}

\bibliographystyle{siamplain}
\bibliography{siam}

\end{document}